\documentclass[12pt, a4paper]{article}

\usepackage[a4paper]{geometry}
\usepackage{pgfplots}
\pgfplotsset{
    x tick style={color=black},
    y tick style={color=black}
}

\usepackage{url}
\usepackage{listings}
\usepackage{amssymb}
\usepackage{graphicx}
\usepackage{algcompatible}
\usepackage{algorithm}
\usepackage{url}
\usepackage{rotating}
\usepackage{booktabs}
\usepackage{subfig}
\usepackage{amsmath}
\usepackage{bbm}

\renewcommand{\labelenumi}{(\alph{enumi})}
\renewcommand\theenumi\labelenumi

\usepackage[english]{babel}

\usepackage[utf8]{inputenc}
\usepackage{xspace}
\usepackage{amsmath,amsthm,amssymb,mathtools}
\usepackage{lmodern}

\usepackage[algo2e,ruled,vlined,linesnumbered]{algorithm2e}

\usepackage{xcolor}
\usepackage{tikz}
\usepackage{graphicx}
\usepackage{soul} 

\allowdisplaybreaks[4]
\newtheorem{theorem}{Theorem}
\newtheorem*{theorem*}{Theorem}
\newtheorem{lemma}[theorem]{Lemma}
\newtheorem*{lemma*}{Lemma}

\newcommand{\oea}{$(1 + 1)$~EA\xspace}

\newcommand{\sibea}{$(\mu+1)$~SIBEA\xspace}

\newcommand{\NSGA}{\mbox{NSGA-II}\xspace}

\newcommand{\om}{\textsc{OneMax}\xspace}
\newcommand{\omm}{\textsc{OneMinMax}\xspace}
\newcommand{\onemax}{\om}
\newcommand{\lo}{\textsc{LeadingOnes}\xspace}
\newcommand{\lotz}{\textsc{LeadingOnesTrailingZeroes}\xspace}
\newcommand{\leadingones}{\lo}

\newcommand{\ojzj}{\textsc{OneJumpZeroJump}\xspace}
\newcommand{\cocz}{\textsc{CountingOnesCountingZeroes}\xspace}

\newcommand{\R}{\ensuremath{\mathbb{R}}}

\newcommand{\N}{\ensuremath{\mathbb{N}}} 

\DeclareMathOperator{\nd}{ND}
\DeclareMathOperator{\sd}{SD}

\DeclareMathOperator{\cDis}{cDis}

\DeclareMathOperator{\inx}{in}
\newcommand{\Vinm}{V_{\inx}^-}
\newcommand{\Vinp}{V_{\inx}^+}

\newcommand{\Rnd}{F_1}
\newcommand{\Rbest}{P_t^{*}}
\newcommand{\Rtp}{P_t^{p}}
\newcommand{\pcross}{p^{**}}

\newcommand{\eps}{\varepsilon} 
\newcommand{\SMS}{\mbox{SMS-EMOA}\xspace}

\let\originalleft\left
\let\originalright\right
\renewcommand{\left}{\mathopen{}\mathclose\bgroup\originalleft}
\renewcommand{\right}{\aftergroup\egroup\originalright}

\usepackage{hyperref}

\begin{document}
\sloppy
\date{}
\title{Mathematical Runtime Analysis for the Non-Dominated Sorting Genetic Algorithm II (\mbox{NSGA-II})\thanks{Extended version of a paper that appeared at AAAI 2022~\cite{ZhengLD22}, and that was conducted when the first author was with Southern University of Science and Technology. This version contains all proofs, many of them revised and improved. In particular, the runtime result for the \mbox{NSGA-II} with tournament selection on \lotz now holds for $N \ge 4(n+1)$ instead of $N \ge 5(n+1)$. In addition, in all upper bounds we now also regard binary tournament selection as in Deb's implementation of the \mbox{NSGA-II} (building the $N$ tournaments from two permutations of the population). We added tail bounds for the runtime guarantees. The experimental section was extended as well. }
}
\author{Weijie Zheng\\International Research Institute for Artificial Intelligence\\
         School of Computer Science and Technology\\
       Harbin Institute of Technology\\
        Shenzhen, China
\and Benjamin Doerr\thanks{Corresponding author.}\\ Laboratoire d'Informatique (LIX)\\ \'Ecole Polytechnique, CNRS\\ Institute Polytechnique de Paris\\ Palaiseau, France
}

\maketitle

\begin{abstract}
The non-dominated sorting genetic algorithm II (\mbox{NSGA-II}) is the most intensively used multi-objective evolutionary algorithm (MOEA) in real-world applications. However, in contrast to several simple MOEAs analyzed also via mathematical means, no such study exists for the \mbox{NSGA-II} so far. In this work, we show that mathematical runtime analyses are feasible also for the \mbox{NSGA-II}. As particular results, we prove that with a population size four times larger than the size of the Pareto front, the \mbox{NSGA-II} with two classic mutation operators and four different ways to select the parents satisfies the same asymptotic runtime guarantees as the SEMO and GSEMO algorithms on the basic \omm and \lotz benchmarks. However, if the population size is only equal to the size of the Pareto front, then the \mbox{NSGA-II} cannot efficiently compute the full Pareto front: for an exponential number of iterations, the population will always miss a constant fraction of the Pareto front. Our experiments confirm the above findings.
\end{abstract}

\section{Introduction}

Many real-world problems need to optimize multiple conflicting objectives simultaneously, see~\cite{ZhouQLZSZ11} for a discussion of the different areas in which such problems arise. Instead of computing a single good solution, a common approach to such multi-objective optimization problems is to compute a set of interesting solutions so that a decision maker can select the most desirable one from these. Multi-objective evolutionary algorithms (MOEAs) are a natural choice for such problems due to their population-based nature. Such multi-objective evolutionary algorithms (MOEAs) have been successfully used in many real-world applications~\cite{ZhouQLZSZ11}. 

Unfortunately, the theoretical understanding of MOEAs falls far behind their success in practice, and this discrepancy is even larger than in single-objective evolutionary computation, where the last twenty years have seen some noteworthy progress on the theory side~\cite{NeumannW10,AugerD11,Jansen13,DoerrN20}. After some early theoretical works on convergence properties, e.g.,~\cite{Rudolph98ep}, the first mathematical runtime analysis of an MOEA was conducted by Laumanns et al.~\cite{LaumannsTZWD02,LaumannsTZ04}. They analyzed the runtime of the \emph{simple evolutionary multi-objective optimizer (SEMO)}, a multi-objective counterpart of the \emph{randomized local search} heuristic, on the \cocz and \lotz benchmarks, which are bi-objective analogues of the classic (single-objective) \onemax and \leadingones benchmark. Around the same time, Giel~\cite{Giel03} analyzed the \emph{global SEMO (GSEMO)}, the multi-objective counterpart of the \oea, on the \lotz function. 

Subsequent theoretical works majorly focused on variants of these algorithms and analyzed their runtime on the \cocz and \lotz benchmarks, on variants of them, on new benchmarks, and on combinatorial optimization problems~\cite{QianYZ13,BianQT18ijcaigeneral,RoostapourNNF19,QianBF20,BianFQY20,DoerrZ21aaai}. We note that the (G)SEMO algorithm keeps all non-dominated solutions in the population and discards all others, which can lead to impractically large population sizes. There are three theory works~\cite{BrockhoffFN08,NguyenSN15,DoerrGN16} on the runtime of a simple hypervolume-based MOEA called \emph{($\mu+1$) simple indicator-based evolutionary algorithm} (($\mu+1$)-SIBEA), regarding both classic benchmarks and problems designed to highlight particular strengths and weaknesses of this algorithm. As the SEMO and GSEMO, the \sibea also creates a single offspring per generation; different from the former, it works with a fixed population size~$\mu$. 

Recently,  also decomposition-based multi-objective evolutionary algorithms were analyzed~\cite{LiZZZ16,HuangZCH19,HuangZ20}. These algorithms decompose the multi-objective problem into several related single-objective problems and then solve the single-objective problems in a co-evolutionary manner. This direction is fundamentally different from the above works and our research. Since not primarily focused on multi-objective optimization, we also do not discuss further the successful line of works that solve constrained single-objective problems by turning the constraint violation into a second objective, see, e.g.,~\cite{NeumannW06ants,FriedrichHHNW10,NeumannRS11,FriedrichN15,QianYZ15,QianSYT17,QianYTYZ19,Crawford19,DoerrDNNS20,Crawford21}.

Unfortunately, most of the algorithms discussed in these theoretical works are far from the MOEAs used in practice. As pointed out in the survey~\cite{ZhouQLZSZ11}, the majority of the MOEAs used in research and applications builds on the framework of the \emph{non-dominated sorting genetic algorithm II (\mbox{NSGA-II})}~\cite{DebPAM02}. 
This algorithm works with a population of fixed size~$N$. It uses a complete order defined by the non-dominated sorting and the crowding distance to compare individuals. In each generation, $N$ offspring are generated from the parent population and the $N$ best individuals (according to the complete order) are selected as the new parent population. This approach is thus substantially different from the (G)SEMO algorithm and hypervolume-based approaches (and naturally completely different from decomposition-based methods), see the features of these algorithms described in the above two paragraphs.
 
Both the predominance in practice and the fundamentally different working principles ask for a rigorous understanding of the \mbox{NSGA-II}. However, to the best of our knowledge so far no mathematical runtime analysis for the \mbox{NSGA-II} has appeared.\footnote{By mathematical runtime analysis, we mean the question of how many function evaluations a black-box algorithm takes to achieve a certain goal. The computational complexity of the operators used by the \mbox{NSGA-II}, in particular, how to most efficiently implement the non-dominated sorting routine, is a different question (and one that is well-understood~\cite{DebPAM02}).} We note that the runtime analysis in~\cite{OsunaGNS20} considers a (G)SEMO algorithm that uses the crowding distance as one of several diversity measures used in the selection of the single parent creating an offspring, but due to the differences of the basic algorithms, none of the arguments used there appear helpful in the analysis of the \mbox{NSGA-II}.   

\textbf{Our Contributions.} This paper conducts the first mathematical runtime analysis of the \mbox{NSGA-II}. We regard the \mbox{NSGA-II} with four different parent selection strategies (choosing each individual as a parent once, choosing parents independently and uniformly at random, and two ways of choosing the parents via binary tournaments) and with two classic mutation operators (one-bit mutation and standard bit-wise mutation), but in this first work without crossover (we remark that crossover is very little understood from the runtime perspective in MOEAs, the only works prior to ours we are aware of are~\cite{NeumannT10,QianYZ13,HuangZCH19}). As previous theoretical works, we analyze how long the \NSGA takes to cover the full Pareto front, that is, we estimate the number of iterations until the parent population contains an individual for each objective value of the Pareto front. 

When trying to determine the runtime of the \NSGA, we first note that the selection mechanism of the \NSGA may remove all individuals with some fixed objective value on the front. In other words, the fact that a certain objective value on the Pareto front was found in some iteration does not mean that this is not lost in some later iteration. This is one of the substantial differences to the (G)SEMO algorithm. We prove that if the population size $N$ is at least four times larger than the size of the Pareto front, then both for the \omm and the \lotz benchmarks, such a loss of Pareto front points cannot occur. With this insight, we then show that each of these eight variants of the \mbox{NSGA-II} computes the full Pareto front of the \omm benchmark in an expected number of $O(n\log n)$ iterations (Theorems~\ref{thm:ommeasy} and~\ref{thm:ommbinary}) and the front of the \lotz benchmark in $O(n^2)$ iterations (Theorems~\ref{thm:lotzeasy} and~\ref{thm:lotzbinary}). When $N = \Theta(n)$, the corresponding runtime guarantees in terms of fitness evaluations, $O(Nn\log n) = O(n^2 \log n)$ and $O(Nn^2)= O(n^3)$, have the same asymptotic order as those proven previously for the SEMO, GSEMO, and \sibea (when $\mu \ge n+1$ and when $\mu = \Theta(n)$ for the \sibea). We note that the benchmarks \omm and \lotz are the two most intensively studied benchmarks in the runtime analysis of MOEAs. In this first runtime analysis work on the \NSGA, we therefore concentrated on these two benchmarks to allow a good comparison with the known performance of other MOEAs.

Using a population size larger than the size of the Pareto front is necessary. We prove that if the population size is equal to the size of the Pareto front, then the \mbox{NSGA-II} (applying one-bit mutation once to each parent) regularly loses points on the Pareto front of \omm. This effect is strong enough so that with high probability for an exponential time each generation of the \mbox{NSGA-II} does not cover a constant fraction of the Pareto front of \omm. 

Our short experimental analysis confirms these findings and gives some quantitative estimates for which mathematical analyses are not precise enough. For example, we observe that also with population sizes smaller than what is required for our theoretical analysis (four times the size of the Pareto front), the \NSGA efficiently  covered the Pareto front of the \omm and \lotz benchmarks. With suitable population sizes, the \NSGA beats the GSEMO algorithm on these benchmarks. Complementing our negative result, we observe that the fraction of the Pareto front not covered when using a population size equal to the front size is around 20\% for \omm and 40\% for \lotz. Also without covering the full Pareto front, MOEAs can serve their purpose of proposing to a decision maker a set of interesting solutions. With this perspective, we also regard experimentally the sets of solutions evolved by the \NSGA when the population size is only equal to the size of the Pareto front. For both benchmarks, we observe that after a moderate runtime, the population contains the two extremal solutions and covers in a very evenly manner the rest of the Pareto front. 

Overall, this work shows that the \mbox{NSGA-II} despite its higher complexity (parallel generation of offspring, selection based on non-dominated sorting and crowding distance) admits mathematical runtime analyses in a similar fashion as done before for simpler MOEAs, which hopefully will lead to a deeper understanding of the working principles of this important algorithm. 

\textbf{Subsequent works:} We note that the conference version~\cite{ZhengLD22} of this work has already inspired a substantial amount of subsequent research. We brief{}ly describe these results now. In~\cite{ZhengD22gecco}, the performance of the \NSGA with small population size was analyzed. The main result is that the problem that Pareto front points can be lost can be significantly reduced with a small modification of the selection procedure that was previously analyzed experimentally~\cite{KukkonenD06}, namely to remove individuals in the selection of the next population sequentially, recomputing the crowding distance after each removal. For this setting, an $O(n/N)$ approximation guarantee was proven. In~\cite{BianQ22}, the first runtime analysis of the \NSGA with crossover was conducted, however, no speed-ups from crossover could be shown. Also, significant speed-ups were shown when using larger tournaments than binary tournaments. In~\cite{DoerrQ23tec}, the performance of the \NSGA on the multimodal benchmark \ojzj benchmark~\cite{DoerrZ21aaai} was analyzed. This work shows that the \NSGA also on this multimodal benchmark has a performance asymptotically at least as good as the GSEMO algorithm (when the population size is at least four times the size of the Pareto front). A matching lower bound for this and our result on \omm was proven in~\cite{DoerrQ23LB}. This work in particular shows that the \NSGA in these settings does not profit from population sizes larger than the minimum required population size. Two recent works showed significant performance gains from crossover, one on the \ojzj benchmark~\cite{DoerrQ23crossover} and one on an artificial problem~\cite{DangOSS23aaai}. The first runtime analysis of the \NSGA on a combinatorial problem, namely the bi-objective minimum spanning tree problem previously regarded in~\cite{Neumann07,NeumannW22}, was conducted in~\cite{CerfDHKW23}. The first runtime analysis of the \NSGA for noisy optimization appeared in~\cite{DangOSS23gecco}. The first runtime analysis of the \SMS~\cite{BeumeNE07} (a variant of the \NSGA building on the hyper-volume) was conducted in~\cite{BianZLQ23}. All these works regard bi-objective problems. For the \omm problem in three or more objectives, it was shown in~\cite{ZhengD22arxivmany} that the \NSGA cannot find the full Pareto front in polynomial time and even has difficulties in approximating it. It was shown in~\cite{WiethegerD23} that the NSGA-III does not experience these problems, at least in three objectives. With this recent development, we are confident to claim that our first mathematical runtime analysis for the \NSGA has started a fruitful direction of research. 

The remainder of the paper is organized as follows. The \mbox{NSGA-II} framework is brief{}ly introduced in Section~\ref{sec:pre}. Sections~\ref{sec:omm} and~\ref{sec:lotz} separately show our runtime results of the \mbox{NSGA-II} with large enough population size on the \omm and \lotz functions. Section~\ref{sec:lb} proves the exponential runtime of the NSGA with population size equal to the size of the Pareto front. Our experimental results are discussed in Section~\ref{sec:exp}. Section~\ref{sec:con} concludes this work.

\section{Preliminaries}
\label{sec:pre}
In this section, we give a brief introduction to multi-objective optimization and to the \mbox{NSGA-II} framework. For the simplicity of presentation, we shall concentrate on two objectives, both of which have to be maximized. A bi-objective objective function on some search space $\Omega$ is a pair $f = (f_1, f_2)$ where $f_i : \Omega \to \R$. We write $f(x) = (f_1(x),f_2(x))$ for all $x \in \Omega$. We shall always assume that we have a bit-string representation, that is, that $S = \{0,1\}^n$ for some $n \in \N$. The challenge in multi-objective optimization is that usually there is no solution that maximizes both $f_1$ and $f_2$ and thus is at least as good as all other solutions. 

More precisely, in bi-objective maximization, we say \emph{$x$ weakly dominates~$y$}, denoted by $x \succeq y$, if and only if $f_1(x) \ge f_1(y)$ and $f_2(x) \ge f_2(y)$. We say \emph{$x$ strictly dominates $y$}, denoted by $x \succ y$, if and only if $f_1(x) \ge f_1(y)$ and $f_2(x) \ge f_2(y)$ and at least one of the inequalities is strict. 
We say that a solution is \emph{Pareto-optimal} if it is not strictly dominated by any other solution. The set of objective values of all Pareto optima is called the \emph{Pareto front} of $f$. With this language, the aim in multi-objective optimization is to compute a small set $P$ of Pareto optima such that $f(P) = \{f(x) \mid x \in P\}$ is the Pareto front or is at least a diverse subset of it. Consider an algorithm $A$ optimizing a multi-objective problem $f$ with Pareto front $M$. Let $P_t$ be the population at iteration $t$ and $G_t$ be the number of function evaluations till iteration $t$, then the \emph{time complexity} or \emph{running time} in this paper is the random variable $T_{A}(f)=\inf\{G_t \mid f(P_t) \supseteq M\}$. Usually, we discuss the expected runtime or the runtime with some probability.

\subsection*{The \mbox{NSGA-II}}

When working with a fixed population size, an MOEA must select the next parent population from the combined parent and offspring population by discarding some of these individuals. For this, a complete order on the combined parent and offspring population could be used so that the next parent population is taken in a greedy manner according to this order. Since dominance is only a partial order, the \mbox{NSGA-II}~\cite{DebPAM02} extends the dominance relation to the following complete order. 

In a given population $P \subseteq \{0,1\}^n$, each individual $x$ has both a rank and a crowding distance. The ranks are defined recursively based on the dominance relation. All individuals that are not strictly dominated by another one have rank one. Given that the ranks $1, \dots, k$ are already defined, the individuals of rank $k+1$ are those among the remaining individuals that are not strictly dominated except by individuals of rank $k$ or smaller. This defines a partition of $P$ into sets $F_1, F_2, \dots$ such that $F_i$ contains all individuals with rank $i$. As shown in~\cite{DebPAM02}, this partition can be computed more efficiently than what the above recursive description suggests, namely in quadratic time, see Algorithm~\ref{alg:nds} for details. It is clear that individuals with lower rank are more interesting, so when comparing two individuals of different ranks, the one with lower rank is preferred. 

To compare individuals in the same rank class $F_i$, the crowding distance of these individuals (in $F_i$) is computed, and the individual with larger distance is preferred. Ties are broken randomly. 

\begin{algorithm}[!ht]
    \caption{fast-non-dominated-sort(S)}
    {\small
    \begin{algorithmic}[1]
     \STATEx \textbf{Input:} $S=\{S_1,\dots,S_{|S|}\}$: the set of individuals
     \STATEx \textbf{Output:} $F_1,F_2,\dots$
    \FOR {$i=1,\dots,|S|$}
    \STATE $\nd(S_i)=0$ \emph{\% number of individuals strictly dominating $S_i$}
    \STATE $\sd(S_i)=\emptyset$ \emph{\% set of individuals strictly dominated by $S_i$}
    \ENDFOR
    \FOR {$i=1,\dots,|S|$} \emph{\% compute $\nd$ and $\sd$}
    \FOR {$j=1,\dots,|S|$}
    \IF {$S_i \prec S_j$}
    \STATE {$\nd(S_i)=\nd(S_i)+1$}
    \STATE {$\sd(S_j)=\sd(S_j)\cup\{S_i\}$}
    \ENDIF
    \ENDFOR
    \ENDFOR
    \STATE $F_1=\{S_i \mid \nd(S_i) =0, i=1,2,\dots,|S|\}$
    \STATE $k=1$
    \WHILE{$F_k \ne \emptyset$}
    \STATE $F_{k+1}=\emptyset$
    \FOR {any $s \in F_k$} \emph{\% discount $F_k$ from $\nd$ and $\sd$}
    \FOR {any $s' \in \sd(s)$}
    \STATE $\nd(s')=\nd(s')-1$
    \IF {$\nd(s')=0$}
    \STATE {$F_{k+1}=F_{k+1}\cup \{s'\}$}
    \ENDIF
    \ENDFOR
    \ENDFOR
    \STATE $k=k+1$
    \ENDWHILE
    \end{algorithmic}
    \label{alg:nds}
    }
\end{algorithm}

Algorithm~\ref{alg:cDis} shows how the crowding distance in a given set $S$ is computed. The crowding distance of some $x \in S$ is the sum of the crowding distances $x$ has with respect to each objective function $f_i$. For a given $f_i$, the individuals in $S$ are sorted in order of ascending $f_i$ value (for equal values, a tie-breaking mechanism is needed, but we shall not make any assumption on this, that is, our mathematical results are valid regardless of how these ties are broken). The first individual and the last individual in the sorted list have an infinite crowding distance. For other individuals in the sorted list, their crowding distance with respect to $f_i$ is the difference of the objective values of its left and right neighbor in the sorted list, normalized by the difference between the first and the last.
\begin{algorithm}[tb]
    \caption{crowding-distance($S$)}
     \textbf{Input:} $S=\{S_1,\dots,S_{|S|}\}$: the set of individuals\\
     \textbf{Output:} $\cDis(S)=(\cDis(S_1),\dots,\cDis(S_{|S|}))$, the vector of crowding distances of the individuals in $S$
		
    \begin{algorithmic}[1]
    \STATE $\cDis(S)=(0,\dots,0)$
    \FOR {each objective function $f_i$}
    \STATE {Sort $S$ in order of ascending $f_i$ value: $S_{i.1},\dots,S_{i.{|S|}}$}
    \STATE {$\cDis(S_{i.1})=+\infty, \cDis(S_{i.{|S|}})=+\infty$}
    \FOR {$j=2,\dots, |S|-1$}
    \STATE {$\cDis(S_{i.j})=\cDis(S_{i.j}) + \frac{f_i(S_{i.{j+1}})-f_i(S_{i.{j-1}})}{f_i(S_{i.{|S|}})-f_i(S_{i.1})}$}
    \ENDFOR
    \ENDFOR
    \end{algorithmic}
    \label{alg:cDis}
\end{algorithm}

The whole \mbox{NSGA-II} framework is shown in Algorithm~\ref{alg:NSGA-II}. After the random initialization of the population of size~$N$, the main loop starts with the generation of $N$ offspring (the precise way how this is done is not part of the \mbox{NSGA-II} framework and is mostly left as a design choice to the algorithm user in~\cite{DebPAM02}, although it is suggested to select parents via binary tournaments based the total order described above). Then the total order based on rank and crowding distance is used to remove the worst $N$ individuals in the union of the parent and offspring population. The remaining individuals form the parent population of the next iteration. 
\begin{algorithm}[!ht]
    \caption{\mbox{NSGA-II}}
    \begin{algorithmic}[1]
    \STATE {Uniformly at random generate the initial population $P_0=\{x_1,x_2,\dots,x_N\}$ with $x_i\in\{0,1\}^n,i=1,2,\dots,N.$}\label{ste:initialize}
    \FOR{$t = 0, 1, 2, \dots$} \label{ste:iterate}
    \STATE {Generate the offspring population $Q_t$ with size $N$}\label{ste:generate}
    \STATE {Use Algorithm~\ref{alg:nds} to divide $R_t=P_t\cup Q_t$ into $F_1,F_2,\dots$}
    \label{ste:sort}
    \STATE {Find $i^* \ge 1$ such that $\sum_{i=1}^{i^*-1}|F_i| < N$ and $\sum_{i=1}^{i^*}|F_i| \ge N$}\label{ste:rank}
    \STATE {Use Algorithm~\ref{alg:cDis} to separately calculate the crowding distance of each individual in $F_{1},\dots,F_{i^*}$}\label{ste:cDis}
    \STATE {Let $\tilde{F}_{i^{*}}$ be the $N-\sum_{i=0}^{i^*-1}|F_{i}|$ individuals in $F_{i^*}$ with largest crowding distance, chosen at random in case of a tie}\label{ste:final front}
    \STATE {$P_{t+1}=\left(\bigcup_{i=1}^{i^*-1}F_i\right)\cup\tilde{F}_{i^*}$}\label{ste:new parents}
    \ENDFOR 
    \end{algorithmic}
    \label{alg:NSGA-II}
\end{algorithm}

\section{Runtime  of the \mbox{NSGA-II} on \omm}\label{sec:omm}

In this section, we analyze the runtime of the \mbox{NSGA-II} on the \omm benchmark proposed first by Giel and Lehre~\cite{GielL10} as a bi-objective analogue of the classic \onemax benchmark. It is the function $f:\{0, 1\}^n \to \N \times \N$ defined by
\[
f(x) = \big(f_1(x), f_2(x)\big) = \big( n - \sum_{i=1}^{n} x_i, \sum_{i=1}^{n} x_i \big)
\]
for all $x = (x_1, \dots, x_n) \in \{0,1\}^n$. The aim is to maximize both objectives in $f$. We immediately note that for this benchmark problem, any solution lies on the Pareto front. It is hence a good example to study how an MOEA explores the Pareto front when already some Pareto optima were found. 

Giel and Lehre~\cite{GielL10} showed that the simple SEMO algorithm finds the full Pareto front of \omm in $O(n^2 \log n)$ iterations and fitness evaluations. Their proof can easily be extended to the GSEMO algorithm. For the SEMO, a (matching) lower bound of $\Omega(n^2 \log n)$ was shown in~\cite{OsunaGNS20}. An upper bound of $O(\mu n \log n)$ was shown for the hypervolume-based \sibea with $\mu \ge n+1$~\cite{NguyenSN15}. When the SEMO or GSEMO is enriched with a diversity mechanism (strong enough so that solutions that can create a new point on the Pareto front are chosen with constant probability), then the runtime of these algorithms reduces to $O(n \log n)$~\cite{OsunaGNS20}. 

In contrast to the SEMO and GSEMO as well as the \sibea with population size $\mu \ge n+1$, the \mbox{NSGA-II} can lose all solutions covering a point of the Pareto front. In the following lemma, central to our runtime analyses on \omm, we show that this cannot happen when the population size is large enough, namely at least four times the size of the Pareto front. Besides, we shall use $[i..j], i\le j$, to denote the set $\{i,i+1,\dots,j\}$ in this paper.

\begin{lemma} 
Consider one iteration of the \mbox{NSGA-II} with population size $N \ge 
4(n+1)$ optimizing the \omm function. Assume that in some iteration~$t$ the combined parent and offspring population $R_t = P_t \cup Q_t$ 
contains a solution $x$ with objective value $(k,n-k)$ for some $k \in 
[0..n]$. Then also the next parent population $P_{t+1}$ contains an 
individual $y$ with $f(y) = (k, n-k)$.
\label{lem:keep}
\end{lemma}

\begin{proof} 

It is not difficult to see that for any $x, y \in \{0, 1\}^n$, we have $x \nprec y$ and $y \nprec x$. Hence, all individuals in $R_t$ have rank one in the non-dominated sorting of~$R_t$, that is, after Step~4 in Algorithm~\ref{alg:NSGA-II}. Thus, in the notation of the algorithm, $F_1 = R_t$ and $i^* = 1$.

We calculate the crowding distance of each individual of $R_t$.  
Let $k \in [0 .. n].$ Assume that there is at least one individual $x \in R_t$ such that $f(x) = (k, n-k)$. 
We recall from Algorithm~\ref{alg:cDis} that $S_{1.1}, \dots , S_{1.{2N}}$ and $S_{2.1}, \dots, S_{2.{2N}}$ are the sorted populations based on $f_1$ and $f_2$ respectively. Since the individuals with the same objective value will continuously appear in the sorted list w.r.t. this objective value, we know that there exist $a \leq b$ and $a' \leq b'$ such that $[a .. b] = \{i \mid f_1(S_{1.i}) = k \}$ and $[a' .. b'] = \{i \mid f_2(S_{2.i}) = n-k \}$. From the crowding distance calculation in Algorithm~\ref{alg:cDis}, we know that $\cDis(S_{1.a}) \ge \frac{f_1\left(S_{1.{a+1}}\right)-f_1\left(S_{1.{a-1}}\right)}{f_1\left(S_{1.{|S|}}\right)-f_1\left(S_{1.1}\right)} \ge \frac{f_1\left(S_{1.{a}}\right)-f_1\left(S_{1.{a-1}}\right)}{f_1\left(S_{1.{|S|}}\right)-f_1\left(S_{1.1}\right)} > 0$ since $f_1\left(S_{1.{a}}\right)-f_1\left(S_{a-1}\right) > 0$ by the definition of $a$. Similarly, we have  $\cDis(S_{1.b}) > 0$, $\cDis(S_{2.a'}) > 0$ and $\cDis(S_{2.b'}) > 0$. For all $j \in [ a+1.. b -1]$ with $S_{1.j} \notin \{ S_{2.a'}, S_{2.b'}\}$, we know $f_1(S_{1.{j-1}}) = f_1(S_{1.{j+1}}) = k$ and $f_2(S_{1.{j-1}}) = f_2(S_{1.{j+1}}) = n-k$ from the definitions of $a, b, a'$ and $ b'$. Hence, we have $\cDis(S_{1.j}) = \frac{f_1(S_{1.{j+1}}) - f_1(S_{1.{j-1})}}{f_1\left(S_{1.{|S|}}\right)-f_1\left(S_{1.1}\right)} + \frac{f_2(S_{2.{j'+1}}) - f_2(S_{2.{j'-1}})}{f_2\left(S_{2.{|S|}}\right)-f_2\left(S_{2.1}\right)} =  0$. 

This shows that the individuals with objective value $(k, n-k)$ and positive crowding distance are exactly $S_{1.a}$, $S_{1.b}$, $S_{2.a'}$ and $S_{2.b'}$. Hence, for each $(k, n-k)$, there are at most four solutions $x$ with $f(x) = (k, n-k)$ and $\cDis(x) > 0$. Noting that the Pareto front size for \omm is $n+1$, the number of individuals with positive crowding distance is at most $4(n+1)\leq N$. Since Step~\ref{ste:final front} in Algorithm~\ref{alg:NSGA-II} keeps $N$ individuals with largest crowding distance, we know that all individuals with positive crowding distance will be kept. Thus, $y=S_{1.a} \in P_{t+1}$, proving our claim.
\end{proof}

Since Lemma~\ref{lem:keep} ensures that objective values on the Pareto front will not be lost in the future, we can estimate the runtime of the \mbox{NSGA-II} via the sum of the waiting times for finding a new Pareto solution. Apart from the fact that the \mbox{NSGA-II} generates $N$ solutions per iteration (which requires some non-trivial arguments in the case of binary tournament selection), this analysis resembles the known analysis of the simpler SEMO algorithm~\cite{GielL10}. For $N = O(n)$, we also obtain the same runtime estimate (in terms of fitness evaluations). 

We start with the easier case that parents are chosen uniformly at random or that each parent creates one offspring.

\begin{theorem}\label{thm:ommeasy}
Consider optimizing the \omm function via the \mbox{NSGA-II} with one of the following four ways to generate the offspring population in Step~\ref{ste:generate} in Algorithm~\ref{alg:NSGA-II}, namely applying one-bit mutation or standard bit-wise mutation once to each parent or $N$ times choosing a parent uniformly at random and applying one-bit mutation or standard bit-wise mutation to it. If the population size $N$ is at least $4(n+1)$, then the expected runtime is at most $\frac{2e^2}{e-1}n(\ln n + 1)$ iterations and at most $\frac{2e^2}{e-1}Nn(\ln n + 1)$ fitness evaluations. Besides, let $T$ be the number of iterations to reach the full Pareto front, then $\Pr[T\ge \tfrac{2e^2(1+\delta)}{e-1} n\ln n] \le 2n^{-\delta}$ holds for any $\delta \ge 0$.
\end{theorem}

\begin{proof}
Let $x \in P_t$ and $f(x) = (k, n-k)$ for some $k \in [0 .. n]$. Let $p$ denote the probability that $x$ is chosen as parent to be mutated. Conditional on that, let $p_k^+$ denote the probability of generating from $x$ an offspring $y_+$ with $f(y_+) = (k+1, n-k-1)$ (when $k < n$) and $p_k^-$ denote the probability of generating from $x$ an offspring $y_-$ with $f(y_-) = (k-1, n-k+1)$ (when $k > 0$). Consequently, the probability that $R_t$ contains an individual $y_+$ with objective value $(k+1, n-k-1)$ is at least $p p_k^+$, and the probability that $R_t$ contains an individual $y_-$ with objective value $(k-1, n-k+1)$ is  at least $p p_k^-$. Since Lemma~\ref{lem:keep} implies that any existing \omm objective value will be kept in the iterations afterwards, we know that the expected number of iterations to obtain $y_+$ (resp.\ $y_-$) once $x$ is in the population is at most $\frac{1}{p p_k^+}$ (resp.\ $\frac{1}{p p_k^-}$). 

Assume that the initial population of the Algorithm~\ref{alg:NSGA-II} contains an $x$ with $f(x) = (k_0, n-k_0)$. Then the expected number of iterations to obtain individuals containing objective values $(k_0, n-k_0), (k_0+1, n-k_0-1), \dots, (n, 0)$ is at most $\sum_{i = k_0}^{n-1} \frac{1}{p p_i ^+}$. Similarly,  the expected number of iterations to obtain individuals containing objective values $(k_0-1, n-k_0+1), (k_0-2, n-k_0-2), \dots , (0, n)$ is at most $\sum_{i = 1}^{k_0} \frac{1}{p p_i^-}$. Consequently, the expected number of iterations to cover the whole Pareto front is at most $\sum_{i = k_0}^{n-1} \frac{1}{p p_i^+} + \sum_{i = 1}^{k_0} \frac{1}{p p_i^-}$.

Now we calculate $p$ for the different ways of selecting parents and $p_k^+$ and $p_k^-$ for the different mutation operations. 
If we apply mutation once to each parent in~$P_t$, we apparently have $p = 1$. If we choose the parents independently at random from~$P_t$, then $p = 1- (1-\frac{1}{N})^N \ge 1 - \frac{1}{e}$. For one-bit mutation, we have $p_k^+ = \frac{n-k}{n} $ and $p_k^- = \frac{k}{n}$. For standard bit-wise mutation, we have $p_k^+ \ge \frac{n-k}{n}(1-\frac{1}{n})^{n-1} \ge \frac{n-k}{en}$ and $p_k^- \ge \frac{k}{n}(1-\frac{1}{n})^{n-1} \ge \frac{k}{en}$ .

From these estimates and the fact that the Harmonic number $H_n = \sum_{i=1}^n \frac 1i$ satisfies $H_n < \ln n +1$, it is not difficult to see that all cases lead to an expected runtime of at most
\begin{align*}
\sum_{i = 0}^{n-1} &{}\frac{1}{p p_i^+} + \sum_{i = 1}^{n} \frac{1}{p p_i^-} \le \sum_{i = 0}^{n-1}\frac{1}{(1 - \frac{1}{e})\frac{n-i}{en}}+\sum_{i = 1}^{n}\frac{1}{(1 - \frac{1}{e})\frac{i}{en}} \\
&={} 2\sum_{i = 1}^{n}\frac{1}{(1 - \frac{1}{e})\frac{i}{en}} < \frac{2e^2}{e-1}n(\ln n + 1)
\end{align*}
iterations, hence at most $\frac{2e^2}{e-1}Nn(\ln n + 1)$ fitness evaluations.

Now we will prove the concentration result. Let $X^+_k$ and $X^-_k$ be independent geometric random variables with success probabilities of $(1 - \frac{1}{e})\frac{n-k}{en}$ and $(1 - \frac{1}{e})\frac{k}{en}$, respectively. Let $T$ be the number of iterations to cover the full Pareto front, and let $Z^+=\sum_{k=0}^{n-1}X^+_k$ and $Z^-=\sum_{k=1}^{n}X^-_k$. Then from the above discussion, we know that $Z:=Z^++Z^-$ stochastically dominates $T$ (see~\cite{Doerr19tcs} for a detailed discussion of how to use stochastic domination arguments in the analysis of evolutionary algorithms).  Let the success probabilities of $X^+_{n-1},X^+_{n-2},\dots,X^+_{0}$ be $q_1^+,\dots,q_{n}^+$, and let $q_1^-,\dots,q_{n}^-$ denote the success probabilities of $X^-_1,X^-_2,\dots,X^-_{n}$. Then we have $q_i^+\ge (1 - \frac{1}{e})\frac1e\frac{i}{n}$ and $q_i^-\ge (1 - \frac{1}{e})\frac1e\frac{i}{n}$ for all $i\in[1..n]$. From a Chernoff bound for a sum of such geometric random variables (\cite[Lemma~4]{DoerrD18}, also found in~\cite[Theorem~1.10.35]{Doerr20bookchapter}), we have that for any $\delta \ge 0$,
\begin{equation*}
\Pr\left[Z^+\ge (1+\delta)\frac{e^2}{e-1} n\ln n\right] \le n^{-\delta}
\end{equation*}
and
\begin{equation*}
\Pr\left[Z^-\ge (1+\delta)\frac{e^2}{e-1} n\ln n\right] \le n^{-\delta}.
\end{equation*}
Hence, we have
\begin{equation*}
\Pr\left[Z\ge (1+\delta)\frac{2e^2}{e-1} n\ln n\right] \le 2n^{-\delta}.
\end{equation*}
Since $Z$ stochastically dominates~$T$, we obtain $\Pr[T\ge \tfrac{2e^2(1+\delta)}{e-1} n\ln n] \le 2n^{-\delta}.$
\end{proof}

We now analyze the performance of the \mbox{NSGA-II} on \omm when selecting the parents via binary tournaments, which is the selection method suggested in the original \mbox{NSGA-II} paper~\cite{DebPAM02}. We regard two variants of this selection method. The most natural one, discussed for example in~\cite{GoldbergD90}, is to conduct $N$ \emph{independent tournaments}. Here the offspring population $Q_t$ is generated by $N$ times independently performing the following sequence of actions: (i)~Select two different individuals $x',x''$ uniformly at random from~$P_t$. (ii)~Select $x$ as the better of these two, that is, the one with smaller rank in $P_t$ or, in case of equality, the one with larger crowding distance in $P_t$ (breaking ties randomly). (iii)~Generate an offspring by mutating~$x$. We note that in some definitions of tournament selection the better individual in a tournament is chosen as winner only with some probability $p > 0.5$, but we do not regard this case any further. We note though that all our mathematical results would remain true in this setting. We also note that sometimes the participants of a tournament are selected ``with replacement''. Again, this would not change our results, but we do not discuss this case any further.

A closer look in Deb's implementation of the \mbox{NSGA-II} (see Revision~1.1.6 available at~\cite{DebNSGAII}), and we are thankful for Maxim Buzdalov (Aberystwyth University) for pointing this out to us, shows that here a different way of selecting the parents is used. In this \emph{two-permutation tournament selection scheme}, two random permutations $\pi_1$ and $\pi_2$ of $P_t$ are generated and then a binary tournament is conducted between $\pi_j(2i-1)$ and $\pi_j(2i)$ for all $i \in [1..\frac N2]$ and $j \in \{1,2\}$ (we assume here that $N$ is even). Of course, this is nothing else than saying that twice a random matching on $P_t$ is generated and the end vertices of each matching edge conduct a tournament. Different from independent tournaments, this selection operator cannot be implemented in parallel. On the positive side, it ensures that each individual takes part in exactly two tournaments, so it treats the individuals in a fairer manner. Also, if there is a unique best individual, then this will surely be selected. As above, in our setting where we do not use crossover, each tournament winner is mutated and these $N$ individuals form the offspring population~$Q_t$.

In the case of binary tournament selection, the analysis is slightly more involved since we need to argue that a desired parent is chosen for mutation with constant probability in one iteration. This is easy to see for a parent at the boundary of the front as its crowding distance is infinite, but less obvious for parents not at the boundary. We note that we need to be able to select such parents since we cannot ensure that the population intersects the Pareto front in a contiguous interval (as can be seen, e.g., from the random initial population). We solve this difficulty in the following three lemmas. 

We use the following notation. Consider some iteration $t$. For $i = 1, 2$, let 
\begin{align*}
v_i^{\min} &= \min \{f_i(x)\mid x \in R_{t}\}, \\
v_i^{\max} &= \max \{f_i(x)\mid x \in R_{t}\}
\end{align*}
denote the extremal objective values in the combined parent and offspring population. Let 
$
V = f(R_t) = \{ (f_1(x), f_2(x)) \mid x\in R_{t} \}
$
denote the set of objective values of the solutions in the combined parent and offspring population $R_t$. We define the set of values such that also the right (left) neighbor on the Pareto front is covered by
\begin{align*}
\Vinp = {} &\{ (v_1, v_2) \in V \mid \exists y \in R_{t} : (f_1(y), f_2(y)) = (v_1 +1, v_2-1)\}, \\
\Vinm = {} & \{ (v_1, v_2) \in V \mid \exists y \in R_{t} : (f_1(y), f_2(y)) = (v_1 -1, v_2+1)\}.
\end{align*} 

\begin{lemma}
For any $(v_1, v_2) \in V \setminus (\Vinp \cap \Vinm )$, there is at least one individual $x \in R_t$ with $f(x) = (v_1, v_2)$ and $\cDis(x) \ge \frac{2}{v_1^{\max} - v_1^{\min}}$. 
\label{lem:potential}
\end{lemma}
\begin{proof}
Let $(v_1, v_2) \in V \setminus (\Vinp \cap \Vinm )$, let $S_{1.1}, \dots , S_{1.{2N}}$ be the sorting of $R_t$ according to $f_1$,  and let $[a .. b] = \{i \in [1 .. 2N] \mid f_1(S_{1.i}) = v_1 \}$. If $v_1 \in \{v_1^{\max}, v_1^{\min}\}$, then by definition of the crowding distance, one individual in $f^{-1}((v_1, v_2))$ has an infinite crowding distance. Otherwise, if
$(v_1, v_2) \in V \setminus \Vinm $, then we have $f_1\left(S_{1.{a+1}}\right)-f_1\left(S_{1.{a-1}}\right) \ge f_1\left(S_{1.{a}}\right)-f_1\left(S_{1.{a-1}}\right) \ge  2$ and thus $\cDis(S_{1.a}) \ge  \frac{f_1\left(S_{1.{a+1}}\right)-f_1\left(S_{1.{a-1}}\right)}{v_1^{\max} - v_1^{\min}} \ge \frac{2}{v_1^{\max} - v_1^{\min}}$.  Similarly, if $(v_1, v_2) \in V \setminus \Vinp$, then $f_1\left(S_{1.{b+1}}\right)-f_1\left(S_{1.{b-1}}\right) \ge f_1\left(S_{1.{b+1}}\right)-f_1\left(S_{1.{b}}\right) \ge 2$ and $  \cDis(S_{1.b}) \ge  \frac{f_1\left(S_{1.{b+1}}\right)-f_1\left(S_{1.{b-1}}\right)}{v_1^{\max} - v_1^{\min}} \ge \frac{2}{v_1^{\max} - v_1^{\min}}$.
\end{proof}

\begin{lemma}
For any $(v_1, v_2) \in \Vinp \cap \Vinm $, there are at most two individuals in $R_t$ with objective value $(v_1, v_2)$ and crowding distance at least $\frac{2}{v_1^{\max} - v_1^{\min}}$.
\label{lem:others}
\end{lemma}

\begin{proof}
Let $(v_1, v_2) \in \Vinp \cap \Vinm $, $[a .. b] = \{i \in [1 .. 2N] \mid f_1(S_{1.i}) = v_1 \}$, and $[a' .. b'] = \{j \in [1 .. 2N] \mid f_2(S_{2.j}) = v_2 \}$. Let $C = \{S_{1.a}, S_{1.b}\} \cup \{S_{2.a'}, S_{2.b'}\}$. If $(R_t \cap f^{-1}((v_1, v_2))) \setminus C$ is not empty, then for any $x \in (R_t \cap f^{-1}((v_1, v_2))) \setminus C $, there exist $i \in [a+1 .. b-1]$ and $j \in [a'+1 .. b'-1]$ such that $x = S_{1.i} = S_{2.j}$. Hence $\cDis(x) = \frac{f_1\left(S_{1.{i+1}}\right)-f_1\left(S_{1.{i-1}}\right)}{v_1^{\max} - v_1^{\min}} + \frac{f_2\left(S_{2.{j+1}}\right)-f_2\left(S_{2.{j-1}}\right)}{v_2^{\max} - v_2^{\min}} = 0$. We thus  know that any individual with crowding distance at least $\frac{2}{v_1^{\max} - v_1^{\min}}$ lies in $C$.

For any $x \in C \setminus (\{S_{1.a}, S_{1.b}\} \cap \{S_{2.a'}, S_{2.b'}\})$, we have $\cDis(x) = \frac{1}{v_1^{\max} - v_1^{\min}}$ or $\cDis(x) = \frac{1}{v_2^{\max} - v_2^{\min}} $. We note that $v_1^{\max} - v_1^{\min} = v_2^{\max} - v_2^{\min}$ since $v_1^{\max} = n- v_2^{\min}$ and $v_1^{\min} = n- v_2^{\max}$. Hence $\cDis(x) < \frac{2}{v_1^{\max} - v_1^{\min}}$. Let now $x \in \{S_{1.a}, S_{1.b}\} \cap \{S_{2.a'}, S_{2.b'}\}$. If $|C| = 1$, then $\cDis(x) = \frac{2}{v_1^{\max} - v_1^{\min}} + \frac{2}{v_2^{\max} - v_2^{\min}} = \frac{4}{v_1^{\max} - v_1^{\min}}$. Otherwise, $\cDis(x) = \frac{1}{v_1^{\max} - v_1^{\min}} + \frac{1}{v_2^{\max} - v_2^{\min}} = \frac{2}{v_1^{\max} - v_1^{\min}}$. Therefore, the number of individuals in $R_t \cap f^{-1}((v_1, v_2))$ with crowding distance at least $\frac{2}{v_1^{\max} - v_1^{\min}}$ is $|\{S_{1.a}, S_{1.b}\} \cap \{S_{2.a'}, S_{2.b'}\}|$, which is at most~$2$.
\end{proof}

\begin{lemma}\label{lem:winner}
Let $N \ge 4(n+1)$. Let $P$ be a parent population in a run of the \mbox{NSGA-II} using independent or two-permutation binary tournament selection optimizing \omm. Let $v = (v_1, n-v_1) \notin f(P)$ be a point on the Pareto front that is not covered by $P$, but a neighbor of $v$ on the front is covered by $P$, that is, there is a $y \in P$ such that $\|f(y) - v\|_\infty = 1$.

In the case of independent tournaments, each of the $N$ tournaments with probability at least $\frac 1N (\frac 16 - \frac{3.5}{N-1})$ selects an individual $x$ with $f(x) = f(y)$.

In the case of two-permutation selection, there are two stochastically independent tournaments each of which with probability at least $\frac 16 - \frac{2.5}{N-1}$ selects an individual $x$ with $f(x) = f(y)$.
\end{lemma}

\begin{proof}
  By Lemma~\ref{lem:potential}, there is an individual $x' \in P$ with $f(x') = f(y)$ and crowding distance at least $\frac{2}{v_1^{\max} - v_1^{\min}}$. We estimate the probability that $x'$ is the winner of a tournament. 
	
	We start with the case of independent tournaments and we regard a particular one of these. With probability $\frac 1N$, the individual $x'$ is chosen as the first participant of the tournament. We condition on this and regard the second individual $x''$ of the tournament, which is chosen uniformly at random from the remaining $N-1$ individuals. We shall argue that with good probability, it has a smaller crowding distance, and thus loses the tournament. 
	
	To this aim, we estimate the number of element $z \in P$ that have a crowding distance of $\frac{2}{v_1^{\max} - v_1^{\min}}$ or more (``large crowding distance''). We treat the individuals differently according to their objective value $w = f(z)$. If $w \in \Vinp \cap \Vinm$, then by Lemma~\ref{lem:others} at most two individuals with this objective value have a large crowding distance. For other objective values $w$, we use the general estimate from the proof of Lemma~\ref{lem:keep} that at most four individuals have this objective value and positive crowding distance. This gives an upper bound of $2 |\Vinp \cap \Vinm| + 4 (|f(P)| - |\Vinp \cap \Vinm|) = 2 |f(P)| + 2 |f(P) \setminus (\Vinp \cap \Vinm)|$ individuals with large crowding distance. 
	
	We note that out of each consecutive three elements $(w_1, n-w_1), (w_1+1,n-w_1-1), (w_1+2,n-w_1-2)$ of the Pareto front, at most two can be in $f(P) \setminus (\Vinp \cap \Vinm)$ -- if all three were in $f(P)$, then the middle one would necessarily be in $\Vinp \cap \Vinm$. Consequently, $|f(P) \setminus (\Vinp \cap \Vinm)| \le 2 \lceil \frac{n+1}{3} \rceil$. With this estimate, our upper bound on the number of individuals with large crowding distance becomes at most $2 (n+1) + 4  \lceil \frac{n+1}{3} \rceil \le \frac{10}{3} (n+1) + \frac 83$, and then excluding the first-chosen individual $x'$, we know that the upper bound estimate for the probability that $x''$ has large crowding distance becomes $\frac{1}{N-1} (\frac{10}{3} (n+1) + \frac 83 - 1) = \frac{1}{N-1} (\frac{10}{3} (n+\frac 34) + \frac{10}{3} \frac{1}{4}+ \frac 53) \le \frac 56 + \frac{1}{N-1} (\frac {10}{3} \frac 14 + \frac 53)$. Consequently, the probability that $x'$ is selected as first participant of the tournament and it wins the tournament is at least
	\[\frac 1N \left(\frac 16 - \frac{2.5}{N-1}\right).\]
		
  For the case of two-permutation tournament selection, we note that there are two independent tournaments (stemming from different permutations) in which $x'$ participates. In both, the partner $x''$ of $x'$ is distributed uniformly in $P_t \setminus \{x'\}$. Hence the above arguments can be applied and we see that with probability at least $\frac 16 - \frac{2.5}{N-1}$, the second participant loses against $x'$.
\end{proof} 

With Lemma~\ref{lem:winner}, we can now easily argue that in a given iteration~$t$, we have a constant probability of choosing at least once a parent that is a neighbor of an empty spot on the Pareto front. This allows to re-use the main arguments of the simpler analyses for the cases that the parents were choosing randomly or that each parent creates one offspring. We note that in the following result, as in any other result in this work, we did not try to optimize the leading constant.

\begin{theorem}
Let $n\ge 4$. Consider optimizing the \omm function via the \mbox{NSGA-II} which creates the offspring population by selecting parents via independent binary tournaments or via the two-permutation approach and applying one-bit or standard bit-wise mutation to these. If the population size $N$ is at least $4(n+1)$, then the expected runtime is at most $\tfrac{200e}{3}n(\ln n + 1)$ iterations and at most $\tfrac{200e}{3}Nn(\ln n + 1)$ fitness evaluations. Besides, let $T$ be the number of iterations to reach the full Pareto front. Then we further have that $\Pr[T\ge \tfrac{200e}{3}(1+\delta) n\ln n] \le 2n^{-\delta}$ holds for any $\delta \ge 0$.
\label{thm:ommbinary}
\end{theorem}

\begin{proof}
  Thanks to Lemma~\ref{lem:winner}, we can essentially follow the arguments of the proof of Theorem~\ref{thm:ommeasy}. Let $y \in P_t$ be such that $f(y)$ is a neighbor of a point on the Pareto front that is not in $f(P_t)$. 
	
	For independent tournaments, by Lemma~\ref{lem:winner} a single tournament will select a parent $x$ with $f(x) = f(y)$, that is, also a neighbor of this uncovered point, with probability at least $\frac 1N (\frac 16 - \frac{2.5}{N-1})$. Hence the probability that at least one such parent is selected in this iteration is 
\begin{align*}
p = 1 - \left(1 - \tfrac 1N (\tfrac 16 - \tfrac{2.5}{N-1})\right)^N \ge 1 - \exp\left(-\tfrac 16 + \tfrac{2.5}{N-1}\right) >0.03,
\end{align*}
where the last inequality uses $N \ge 20$ from $n \ge 4$ and $N\ge 4(n+1)$.
	
	For two-permutation tournament selection, again by Lemma~\ref{lem:winner}, with probability at least $p = 1 - (1 - (\frac 16 - \frac{2.5}{N-1}))^2 >0.03 $ (since $N \ge 20$) a parent $x$ with $f(x) = f(y)$ is selected.
	
	With these values of $p$, the proof of Theorem~\ref{thm:ommeasy} extends to the two cases of tournament selection, and we know that the expected iterations to cover the full Pareto front is at most
\begin{align*}
\sum_{i = 0}^{n-1} &{}\frac{1}{p p_i^+} + \sum_{i = 1}^{n} \frac{1}{p p_i^-} \le \sum_{i = 0}^{n-1}\frac{1}{0.03\frac{n-i}{en}}+\sum_{i = 1}^{n}\frac{1}{0.03\frac{i}{en}} \\
&={} 2\sum_{i = 1}^{n}\frac{1}{0.03\frac{i}{en}} < \tfrac{200e}{3}n(\ln n + 1).
\end{align*}
	
We now discuss the concentration result. With the same arguments as in the proof of Theorem~\ref{thm:ommeasy}, but using the success probabilities $0.03\frac{n-k}{en}$ and $0.03\frac{k}{en}$ for $X^+_k$ and $X^-_k$ respectively and estimating $q_i\ge \frac{0.03}{e}\frac{i}{n}$, we obtain that for any $\delta \ge 0$, we have $\Pr[T\ge \tfrac{200e}{3}(1+\delta) n\ln n] \le 2n^{-\delta}.$
\end{proof}

\section{Runtime of the \mbox{NSGA-II} on \lotz}\label{sec:lotz}

We proceed with analyzing the runtime of the \mbox{NSGA-II} on the benchmark \lotz proposed by Laumanns, Thiele, and Zitzler~\cite{LaumannsTZ04}. This is the function $f:\{0, 1\}^n \to \N \times \N$ defined by
\[
f(x) = \big(f_1(x), f_2(x)\big) = \big( \sum_{i=1}^{n} \prod_{j=1}^i x_j, \sum_{i=1}^{n} \prod_{j=i}^n (1-x_j) \big)
\]
for all $x \in \{0,1\}^n$. Here the first objective is the so-called \leadingones function, counting the number of (contiguous) leading ones of the bit string, and the second objective counts in an analogous fashion the number of trailing zeros. Again, the aim is to maximize both objectives. Different from \omm, here many solutions exist that are not Pareto optimal. The known runtimes for this benchmark are $\Theta(n^3)$ for the SEMO~\cite{LaumannsTZ04}, $O(n^3)$ for the GSEMO~\cite{Giel03}, and $O(\mu n^2)$ for the \sibea with population size $\mu \ge n+1$~\cite{BrockhoffFN08}.

Similar to \omm, we can show that when the population size is large enough, an objective value on the Pareto front stays in the population from the point on when it is discovered.

\begin{lemma} 
Consider one iteration of the \mbox{NSGA-II} with population size $N \ge 4(n+1)$ optimizing the \lotz function. Assume that in some iteration 
$t$ the combined parent and offspring population $R_t = P_t \cup Q_t$ 
contains a solution $x$ with rank one. Then also the next parent population $P_{t+1}$ contains an individual $y$ with $f(y) = f(x)$. In particular, once the parent population contains an individual with objective value $(k, n-k)$, it will do so for all future generations.
\label{lem:keeplotz}
\end{lemma}

\begin{proof}
Let $\Rnd$ be the set of solutions of rank one, that is, the set of solutions in $R_t$ that are not dominated by any other individual in $R_t$. By definition of dominance, for each $v_1 \in \{f_1(x) \mid x \in \Rnd\}$, there exists a unique $v_2$ such that $(v_1, v_2) \in \{f(x) \mid x \in \Rnd\}$. Therefore, $|f(\Rnd)|$ is at most $n+1$. We now reuse the argument from the proof of Lemma~\ref{lem:keep} for \omm that for each objective value, there are at most $4$ individuals with this objective value and positive crowding distance. Thus the number of individuals in $\Rnd$ with positive crowding distance is at most $4(n+1)\leq N$. Since the \mbox{NSGA-II} keeps $N$ individuals with smallest rank and largest crowding distance in case of a tie, we know that the individuals with rank one and positive crowding distance will all be kept. This shows the first claim. 

For the second claim, let $x \in P_t$ with $f(x) = (k, n-k)$ for some $k$. Since $x$ lies on the Pareto front of \lotz, the rank of $x$ in $R_t$ is necessarily one. Hence by the first claim, a $y$ with $f(y) = f(x)$ will be contained in $P_{t+1}$. A simple induction extends this finding to all future generations. 
\end{proof}

Since not all individuals are on the Pareto front, the runtime analysis for \lotz function is slightly more complex than for \omm. We analyze the process in two stages: the first stage lasts until we have found both extremal solutions of the Pareto front. In this phase, we argue that the first (resp.~second) objective value increases by one every (expected) $O(n)$ iterations. Consequently, after an expected number of $O(n^2)$ iterations, we have an individual $x$ in the population with $f_1(x)=n$ (resp.~$f_2(x)=n$), which are the desired extremal individuals. The second stage, where we complete the Pareto front from existing Pareto solutions, can be analyzed in a similar manner as for \omm in Theorem~\ref{thm:ommeasy}, noting of course the different probabilities to generate a new solution on the Pareto front. We start with the two easier parent selections and discuss tournament selection separately in Theorem~\ref{thm:lotzbinary}.

\begin{theorem}
Consider optimizing the \lotz function via the \mbox{NSGA-II} with one of the following four ways to generate the offspring population in Step~\ref{ste:generate} in Algorithm~\ref{alg:NSGA-II}, namely applying one-bit mutation or standard bit-wise mutation once to each parent or $N$ times choosing a parent uniformly at random and applying one-bit mutation or standard bit-wise mutation to it. If the population size $N$ is at least $4(n+1)$, then the expected runtime is $\frac{2e^2}{e-1} n^2$ iterations and $\frac{2e^2}{e-1} Nn^2$ fitness evaluations. Besides, let $T$ be the number of iterations to reach the full Pareto front. Then 
\begin{align*}
\Pr\left[T\ge \frac{2e^2(1+\delta)}{e-1} n^2\right] \le \exp\left(-\frac{\delta^2}{2(1+\delta)}(2n-1)\right)
\end{align*}
holds for any $\delta \ge 0$.
\label{thm:lotzeasy}
\end{theorem}

\begin{proof}
Consider one iteration $t$ of the first stage, that is, we have  $\Rtp = \{ x \in P_t \mid f_1(x) + f_2(x) = n\} = \emptyset$. Let $v_t = \max \{ f_1(x) \mid x \in P_t \}$ and let $\Rbest = \{x \in P_t \mid f_1(x) = v_t \}$. Note that by Lemma~\ref{lem:keeplotz}, $v_t$ is non-decreasing over time. Let $x \in \Rbest$. Let $p$ denote the probability that $x$ is chosen as a parent to be mutated (note that this probability is independent of $x$ for the two selection schemes regarded here). Conditional on that, let $p^*$ be a lower bound (independent of $x$) on the probability that $x$ generates a solution with a larger $f_1$-value. Then the expected number of iterations to obtain a solution with better $f_1$-value is at most $\frac{1}{p p^*}$. Consequently, the expected number of iterations to obtain a $f_1$-value of $n$, thus a solution on the Pareto front, is at most $(n-k_0) \frac{1}{p p^*} \le n\frac{1}{pp^*}$, where $k_0$ is the maximum \lo value in the initial population.

For the second stage, let $x \in P_t^p$ be such that a neighbor of $f(x)$ on the front is not yet covered by $P_t$. Let $p'$ denote the probability that $x$ is chosen as a parent to be mutated. Conditional on that, let $\pcross$ denote a lower bound (independent of $x$) for the probability to generate a particular neighbor of $x$ on the front. Consequently, the probability that $R_t$ covers an extra element of the Pareto front is at least $p' \pcross$. Since Lemma~\ref{lem:keeplotz} implies that any existing \lotz value on the front will be kept in the following iterations, we know that the expected number of iterations for this progress is at most $\frac{1}{p' \pcross}$. Since $n$ such progresses are sufficient to cover the full Pareto front, the expected number of iterations to cover the whole Pareto front is at most $n\frac{1}{p' \pcross}$. Therefore, the expected total runtime is at most $n \frac{1}{p p^*} + n\frac{1}{p' \pcross}$ iterations.

We recall from Theorem~\ref{thm:ommeasy} that we have $p = p' = 1$ when selecting each parent once and we have $p = p' = 1- (1-\frac{1}{N})^N \ge 1 - \frac{1}{e}$ when choosing parents randomly. To  estimate $p^*$ and $\pcross$, we note that the desired progress can always be obtained by flipping one particular bit. Hence for one-bit mutation, we have $p^* = \pcross =  \frac{1}{n}$. For standard bit-wise mutation, $\frac{1}{n}(1-\frac{1}{n})^{n-1} \ge \frac{1}{en}$ is a valid lower bound for $p^*$ and $\pcross$.

With these estimates, we obtain in all cases an expected runtime of 
at most
\begin{align*}
n \frac{1}{p p^*} + n\frac{1}{p' \pcross} \le n \frac{1}{(1-\frac 1e)\frac{1}{en}} +  n \frac{1}{(1-\frac 1e)\frac{1}{en}} = \frac{2e^2n^2}{e-1}
\end{align*}
iterations, hence $\frac{2e^2}{e-1} Nn^2$ fitness evaluations.

Now we will prove the concentration result. The time to cover the full Pareto front is divided into two stages as discussed before. It is not difficult to see that the first stage is to reach a Pareto optimum for the first time, and the corresponding runtime is dominated by the sum of $n$ independent geometric random variables with success probabilities of $(1 - \frac{1}{e})\frac{1}{en}$. The second stage is to cover the full Pareto front, and the corresponding runtime is dominated by the sum of another $n$ such independent geometric random variables. Formally, let $X_1,\dots,X_{2n}$ be independent geometric random variables with success probabilities of $(1 - \frac{1}{e})\frac{1}{en}$, and let $T$ be the number of iterations to cover the full Pareto front. Then $Z:=\sum_{i=1}^{2n}X_i$ stochastically dominates $T$, and $E[Z]=\frac{2e^2n^2}{e-1}$. From a Chernoff bound for sums of independent identically distributed geometric random variables~\cite[(1.10.46) in Theorem~1.10.32]{Doerr20bookchapter}, we have that for any $\delta \ge 0$,
\begin{equation*}
\Pr\left[Z\ge (1+\delta)\frac{2e^2n^2}{e-1}\right] \le \exp\left(-\frac{\delta^2}{2}\frac{2n-1}{1+\delta}\right). 
\end{equation*}
Since $Z$ dominates $T$, we have proven this theorem.
\end{proof}

We now study the runtime of the \mbox{NSGA-II} using binary tournament selection. Compared to \omm, we face the additional difficulty that now rank one solutions can exist which are not on the Pareto front. Due to their low rank, they could perform well in the selection, but being possibly far from the front, they are not interesting as parents. We need a sightly different general proof outline to nevertheless argue that sufficiently often a parent on the Pareto front generates a new neighbor on the front. Also, since not all individuals are on the Pareto front, we do not have anymore the property that the difference between the maximum and minimum value is the same for both objectives. We overcome this by first showing the \mbox{NSGA-II} finds the two extremal points of the Pareto front in reasonable time (then the maximum values are both $n$ and the minimum values are both~$0$).

\begin{theorem}
Consider optimizing the \lotz function via the \mbox{NSGA-II}. Assume that the parents for variation are chosen either via $N$ independent random tournaments between different individuals or via the two-permutation implementation of binary tournaments. Assume that these parents are mutated via one-bit or standard bit-wise mutation. If the population size $N$ is at least $4(n+1)$, then the expected runtime is at most $15en^2$ iterations and at most $15eNn^2$  fitness evaluations. Besides, let $T$ be the number of iterations to reach the full Pareto front, then 
\begin{align*}
\Pr\left[T\ge \frac{(1+\delta)100e}{3} n^2\right] \le \exp\left(-\frac{\delta^2}{2(1+\delta)}(3n-1)\right)
\end{align*}
holds for any $\delta \ge 0$.
\label{thm:lotzbinary}
\end{theorem}

\begin{proof}
  We first argue that, regardless of the initial state of the \mbox{NSGA-II}, it takes $O(n^2)$ iterations until the extremal point $(1, \dots, 1)$, which is the unique maximum of $f_1$, is in $P_t$. To this aim, let $X_t := \max\{f_1(x) \mid x \in P_t\}$ denote the maximum $f_1$ value in the parent population. We note that any $x \in P_t$ with $f_1(x) = X_t$ lies on the first front $F_1$ and that there is a $y \in P_t$ with infinite crowding distance and $f(y) = f(x)$, in particular, $f_1(y) = X_t$. 
	
	If parents are chosen via independent tournaments, such a $y$ has a $\frac 2N$ chance of being one of the two individuals of a fixed tournament. It then wins the tournament with at least 50\% chance (where the 50\% show up only in the rare case that the other individual also lies on the first front and has an infinite crowding distance). Hence the probability that this $y$ is chosen at least once as a parent to be mutated is at least $p = 1-(1-\frac 12 \frac{2}{N})^N \ge 1 - \frac 1e$. 
	
	When the two-permutation implementation of tournament selection is used, then $y$ appears in both permutations and has a random partner in both. Again, this partner with probability at most $\frac 12$ wins the tournament. Hence the probability that $y$ is selected as a parent at least once is at least $p = 1 - (\frac 12)^2 = \frac 34$. 
	
	Conditional on $y$ being chosen at least once, let us regard a fixed mutation step in which $y$ was selected as a parent. To mutate $y$ into an individual with higher $f_1$ value, it suffices to flip a particular single bit (namely the first zero after the initial contiguous segment of ones). The probability for this is $p^* = \frac 1n$ for one-bit mutation and $p^* = \frac 1n (1-\frac 1n)^{n-1} \ge \frac 1 {en}$ for standard bit-wise mutation. Denoting by $Y_t := \max\{f_1(x) \mid x \in R_t\}$ the maximum $f_1$ value in the combined parent and offspring population, we have just shown that $\Pr[Y_t \ge X_t + 1] \ge p p^* = \Omega(1/n)$ whenever $X_t < n$. We note that any $x \in R_t$ with $f_1(x) = Y_t$ lies on the first front $F_1$ of $R_t$ and that there is a $y \in R_t$ with infinite crowding distance and $f(y) = f(x)$, in particular, $f_1(y) = Y_t$. Consequently, such a $y$ will be kept in the next parent population $P_{t+1}$ (note that there are at most $4$ individuals in $F_1$ with infinite crowding distance -- since $N \ge 4$, they will all be included in $P_{t+1}$). This shows that we also have $\Pr[X_{t+1} \ge X_t + 1] \ge p p^*$ whenever $X_t < n$. By adding the expected waiting times for an increase of the $X_t$ value, we see that the expected time to have $X_t = n$, that is, to have $(1, \dots, 1) \in P_t$, is at most
\begin{align*}
\frac{n}{p p^*} \le \frac{n}{(1-\frac 1e)\frac{1}{en}} =\frac{e^2n^2}{e-1}
\end{align*}
iterations. 
		
	By a symmetric argument, we see that after another at most $\frac{e^2}{e-1}n^2$ iterations, also the other extremal point $(0, \dots, 0)$ is in the population (and remains there forever by Lemma~\ref{lem:keeplotz}). 

With now both extremal points of the Pareto front covered, we analyze the remaining time until the Pareto front is fully covered. We note that by Lemma~\ref{lem:keeplotz}, the number of Pareto front points covered cannot decrease. Hence it suffices to prove a lower bound for the probability that the coverage increases in one iteration. This is what we do now.

Assume that the Pareto front is not yet fully covered. Since we have some Pareto optimal individuals, there also is a Pareto optimal individual $x \in P_t$ such that $f(x)$ is a neighbor of a point $v$ on the Pareto front that is not covered. Since we have both extremal points in the Pareto front, the differences between the maximum and minimum value are the same for both objectives (namely $n$). Consequently, in the same way as in the proof of Lemma~\ref{lem:potential}, we know that there is also such a $y$ with $f(y)=f(x)$ and with crowding distance at least $\frac{2}{n}$.

We estimate the number of individuals in $P_t \setminus \{y\}$ which could win a tournament against this $y$. Clearly, these can only be individuals in the first front $F_1$ of the non-dominated sorting. Assume first that $|f(F_1)| \le 0.8 (n+1)$. We note that, just by the definition of crowding distance and in a similar fashion as in the proof of Lemma~\ref{lem:keep}, for each $v \in f(F_1)$ there are at most four individuals with $f$ value equal to~$v$ and positive crowding distance. All other individuals in $F_1$ have a crowding distance of zero (and thus lose the tournament against~$y$), as do all individuals not in~$F_1$. Consequently, there are at least $N_0 = N - 4 \cdot 0.8 (n+1)$ individuals other than $y$ that would lose a tournament against~$y$. 

Assume now that $m := |f(F_1)| > 0.8 (n+1)$. Since $F_1$ consists of pair-wise incomparable solutions or solutions with identical objective value (which we may ignore for the following argument), we have $|f_1(F_1)| = |f_2(F_1)| = m$. For any $v = (v_1, v_2)$, let $v^+ := (v_1+1,v_2-1)$ and $v^-:=(v_1-1, v_2+1)$. Then we divide $f(F_1)$ into two disjoint sets $U_1=\{v \in f(F_1) \cap [1..n-1]^2 \mid v^+ \notin f(F_1) \text{ or } v^- \notin f(F_1)\}$ and $U_2 = f(F_1) \setminus U_1$. Since both $f_1(F_1)$ and $f_2(F_1)$ are subsets of $[0..n]$, which has $n+1$ elements, we see that less than $0.2(n+1)$ of the values in $[0..n]$ are missing in $f_1(F_1)$, and analogously in $f_2(F_1)$. Since each value missing in $f_1(F_1)$ or $f_2(F_1)$ leads to at most two values in $U_1$, we have $|U_1|<4 \cdot 0.2(n+1) = 0.8(n+1)$. For the values in $U_1$, we use the blunt estimate from above that at most $4$ individuals with this objective value and positive crowding distance exist. For the values $v\in U_2$, we are in the same situation as in Lemma~\ref{lem:others}, and thus there are at most two individuals $x \in F_1$ with $f(x) = v$ and crowding distance at least $\frac{2}{n}$ (this was not formally proven in Lemma~\ref{lem:others} for the case that $v \in \{(0,n),(n,0)\}$ and the unique neighbor of $v$ is in $f(F_1)$, but it is easy to see that in this case only the at most two $x$ with $f(x)=v$ and infinite crowding distance can have a crowding distance of at least $\frac 2n$). Consequently, there are more than
\begin{align*}
N-4|U_1|-2|U_2|&={}N-4|U_1|-2(m-|U_1|)=N-2|U_1|-2m \\
&>{} N-2\cdot0.8(n+1)-2(n+1)=N-3.6(n+1)
\end{align*}
individuals in $P_t \setminus \{y\}$ that would lose against~$y$. Note that this bound is weaker than the one from the first case, so it is valid in both cases. 

From this, we now estimate the probability that $y$ is selected as a parent at least once. We first regard the case of independent tournaments. The probability that $y$ is the winner of a fixed tournament is at least the probability that it is chosen as the first contestant times the probability that one of the at least $N - 3.6(n+1)$ sure losers is chosen as the second contestant. This probability is at least $\frac 1N \cdot \frac{N - 3.6(n+1)}{N-1} \ge \frac 1N \frac{0.4(n+1)}{4n+3} \ge 0.1 \frac 1N$. Hence the probability $p$ that $y$ is chosen at least once as a parent for mutation is at least $p \ge 1 - (1 - 0.1 \frac 1N)^N \ge 1 - \exp(-0.1) \ge 0.09$. For the two-permutation implementation of tournament selection, $y$ appears in both permutations and has a random partner in each of them. Hence the probability that $y$ wins at least one of these two tournaments is at least $p \ge 1 - (1 - \frac{N - 3.6(n+1)}{N-1})^2 \ge 1-(1-0.1)^2 = 0.19$. 

Conditional on $y$ being selected at least once, we regard a mutation step in which $y$ is selected. The probability $p^*$ that the Pareto optimal $y$ is mutated into the unique Pareto optimal bit string $z$ with $f(z) = v$ is $p^*= \frac 1n$ for one-bit mutation and $p^* = \frac 1n (1-\frac 1n)^{n-1} \ge \frac 1{en}$ for standard bit-wise mutation. Consequently, the probability that one iteration generates the missing Pareto front value $v$ is at least $p p^*$, the expected waiting time for this is at most $\frac{1}{p p^*}$ iterations, and the expected time to create all missing Pareto front values is at most 
\begin{align*}
\frac{n}{p p^*}  \le \frac{n}{0.09 \frac{1}{en}} = \frac{100en^2}{9}
\end{align*}
iterations. Hence, the runtime for the full coverage of the Pareto front starting from the initial population is at most
\begin{align*}
\tfrac{e^2}{e-1}n^2+\tfrac{e^2}{e-1}n^2+ \tfrac{100e}{9}n^2 < 15en^2
\end{align*}
iterations, which is at most $15eNn^2$ fitness evaluations.

Now we will prove the concentration result. Note that in this proof, we consider three phases, the first phase to reach the extremal point $(1,\dots,1)$, the second phase to reach $(0,\dots,0)$, and the third phase to cover the full Pareto front. The runtime for each phase is dominated by the sum of $n$ independent geometric random variables with success probabilities of $\frac{0.09}{en}$. Formally, let $X_1,\dots,X_{3n}$ be independent geometric random variables with success probabilities of $\frac{0.09}{en}$, and let $T$ be the number of iterations to cover the full Pareto front. Then we have $Z:=\sum_{i=1}^{3n}X_i$ stochastically dominates $T$, and $E[Z]=\frac{3en^2}{0.09}$. From the Chernoff bound~\cite[(1.10.46) in Theorem~1.10.32]{Doerr20bookchapter}, we have that for any $\delta \ge 0$,
\begin{equation*}
\Pr\left[Z\ge (1+\delta)\frac{100e}{3}n^2\right] \le \exp\left(-\frac{\delta^2}{2}\frac{3n-1}{1+\delta}\right). 
\end{equation*}
Since $Z$ dominates $T$, this shows the theorem.
\end{proof}

\section{An Exponential Lower Bound for Small Population Size}\label{sec:lb}

In this section, we prove a lower bound for a small population size. Since lower bound proofs can be quite complicated -- recall for example that there are matching upper and lower bounds for the runtime of the SEMO (using one-bit mutation) on \omm and \lotz, but not for the GSEMO (using bit-wise mutation) -- we restrict ourselves to the simplest variant using each parent once to generate one offspring via one-bit mutation. From the proofs, though, we are optimistic that our results, with different implicit constants, can also be shown for all other variants of the \mbox{NSGA-II} regarded in this work. Our experiments support this believe, see Figure~\ref{fig:ratiosVars} in Section~\ref{sec:exp}.

Our main result is that this \mbox{NSGA-II} takes an exponential time to find the whole Pareto front (of size $n+1$) of \omm when the population size is $n+1$. This is different from the SEMO and GSEMO algorithms (which have no fixed population size, but which will never store a population larger than $n+1$ when optimizing \omm) and the \sibea with population size $\mu = n+1$. Even stronger, we show that there is a constant $\eps > 0$ such that when the current population $P_t$ covers at least $|f(P_t)| \ge (1-\eps)(n+1)$ points on the Pareto front of \omm, then with probability $1 - \exp(-\Theta(n))$, the next population $P_{t+1}$ will cover at most $|f(P_{t+1})| \le (1-\eps)(n+1)$ points on the front. Hence when a population covers a large fraction of the Pareto front, then with very high probability the next population will cover fewer points on the front. 
 When the coverage is smaller, that is, $|f(P_t)| \le (1-\eps)(n+1)$, then with probability $1 - \exp(-\Theta(n))$ the combined parent and offspring population $R_t$ will miss a constant fraction of the Pareto front.
From these two statements, it is easy to see that there is a constant $\delta$ such that with probability $1 - \exp(-\Omega(n))$, in none of the first $\exp(\Omega(n))$ iterations the combined parent and offspring population covers more than $(1-\delta)(n+1)$ points of the Pareto front.

Since it is the technically easier one, we start with proving the latter statement that a constant fraction of the front not covered by $P_t$ implies also a constant fraction not covered by $R_t$. Before stating the formal result and proof, let us explain the reason behind this result. With a constant fraction of the front not covered by~$P_t$, also a constant fraction that is $\Omega(n)$ away from the boundary points $(0,n)$ and $(n,0)$ is not covered. These values have the property that from an individual corresponding to either of their neighboring positions, an individual with this objective value can only be generated with constant probability via one-bit mutation. Again a constant fraction of these values have only a constant number of individuals on neighboring positions. These values thus have a (small) constant probability of not being generated in this iteration. This shows that in expectation, we are still missing a constant fraction of the Pareto front in $R_t$. Via the method of bounded differences (exploiting that each mutation operation can change the number of missing elements by at most one), we turn this expectation into a bound that holds with probability $1 - \exp(-\Omega(n))$.

\begin{lemma}
Let $\eps \in (0,1)$ be a sufficiently small constant. Consider optimizing the \omm benchmark via the \mbox{NSGA-II} applying one-bit mutation once to each parent individual. Let the population size be $N=n+1$. Assume that $|f(P_t)| \le (1-\eps)(n+1)$. Then with probability at least $1-\exp(-\Omega(n))$, we have $|f(R_t)| \le (1 - \frac 1 {10} \eps (\tfrac15\eps - \tfrac 2n)^{5/\eps}) (n+1)$.
\label{lem:Rt}
\end{lemma}

\begin{proof}
Let $F = \{(v,n-v) \mid v \in [0..n]\}$ be the Pareto front of \omm. For a value $(v,n-v) \in F$, we say that $(v-1,n-v+1)$ and $(v+1,n-v-1)$ are neighbors of $(v,n-v)$ provided that they are in $[0..n]^2$. We write $(a,b) \sim (u,v)$ to denote that $(a,b)$ and $(u,v)$ are neighbors.

Let $\Delta = \lceil \frac 5 \eps \rceil - 1$ and let $F'$ be the set of values in $F$ such that more than $\Delta$ individuals in $P_t$ have a function value that is a neighbor of this value, that is,
\[
F' = \left\{(v,n-v) \in F \,\big|\, |\{x \in P_t \mid f(x) \sim (v,n-v)\}| \ge \Delta+1\right\}.
\]
Then $|F'| \le \frac 2 {\Delta+1} (n+1) \le \frac 25 \eps(n+1)$ as otherwise the number of individuals in our population could be bounded from below by 
\[
|F'| \tfrac 12 (\Delta+1) > \tfrac 2 {\Delta+1} (n+1) \cdot \tfrac 12 (\Delta+1) = n+1,
\] 
which contradicts our assumption $N=n+1$ (note that the factor of $\tfrac12$ accounts for the fact that we may count each individual twice). 

Let $M= F \setminus f(P_t)$ be the set of Pareto front values not covered by the current population. By assumption, $|M| \ge \eps(n+1)$. Let 
\[
M_1=\left\{(v,n-v) \in M \,\middle|\, v \in \left[\lfloor\tfrac15 \eps (n+1)\rfloor..n-\lfloor\tfrac15 \eps (n+1)\rfloor\right]\right\} \setminus F'.
\]
Then $|M_1| \ge |M| - 2 \lfloor\tfrac15 \eps (n+1)\rfloor- |F'| \ge \tfrac15 \eps (n+1)$.

We now argue that a constant fraction of the values in $M_1$ is not generated in the current generation. We note that via one-bit mutation, a given $(v,n-v) \in F$ can only be generated from an individual $x$ with $f(x) \sim (v,n-v)$. Let $(v,n-v) \in M_1$. Since $v \in [\lfloor\tfrac15 \eps (n+1)\rfloor..n-\lfloor\tfrac15 \eps (n+1)\rfloor]$, the probability that a given parent $x$ is mutated to some individual $y$ with $f(y) = (v,n-v)$ is at most 
\begin{align*}
\frac{n-\lfloor\tfrac15 \eps (n+1)\rfloor+1}{n} \le 1 - \frac15 \eps + \frac2n
\end{align*} 
since there are at most $n-\lfloor\tfrac15 \eps (n+1)\rfloor+1$ bit positions such that flipping them creates the desired value. Since $v \notin F'$, the probability that $Q_t$ (and thus $R_t$) contains no individual $y$ with $f(y) = (v,n-v)$, is at least
\begin{align*}
\left(1-\left(1 - \tfrac15\eps + \tfrac2n\right)\right)^{\Delta} \ge (\tfrac15\eps - \tfrac 2n)^{5/\eps} := p.
\end{align*} 
Let $X = |F \setminus f(R_t)|$ denote the number of Pareto front values not covered by $R_t$. We have $E[X] \ge |M_1|p \ge \frac 15 \eps p (n+1)$. The random variable $X$ is functionally dependent on the $N = n+1$ random decisions of the $N$ mutation operations, which are stochastically independent. Changing the outcome of a single mutation operation changes $X$ by at most $1$. Consequently, $X$ satisfies the assumptions of the method of bounded differences~\cite{McDiarmid89} (also to be found in~\cite[Theorem~1.10.27]{Doerr20bookchapter}). Hence the classic additive Chernoff bound applies to $X$ as if it was a sum of $N$ independent random variables taking values in an interval of length~$1$. In particular, the probability that $X \le \frac 1{10} \eps p (n+1) \le \frac 12 E[X]$ is at most $\exp(-\Omega(n))$.
\end{proof}

We now turn to the other main argument, which is that when the current population covers the Pareto front to a large extent, then the selection procedure of the \mbox{NSGA-II} will remove individuals in such a way from $R_t$  that at least some constant fraction of the Pareto front is not covered by~$P_{t+1}$. The key arguments to show this claim are the following. When a large part of the front is covered by~$P_t$, then many points are only covered by a single individual (since the population size equals the size of the front). With some careful counting, we derive from this that close to two thirds of the positions on the front are covered exactly twice in the combined parent and offspring population $R_t$ and that the corresponding individuals have the same crowding distance. Since these are roughly $\frac 43 (n+1)$ individuals appearing equally preferable in the selection, a random set of at least roughly $\frac 13 (n+1)$ of them will be removed in the selection step. In expectation, this will remove both individuals from a constant fraction of the points on the Pareto front. Again, the method of bounded differences turns this expectation into a statement with probability $1 - \exp(-\Omega(n))$.

\begin{lemma}
Let $\eps > 0$ be a sufficiently small constant. Consider optimizing the \omm benchmark via the \mbox{NSGA-II} applying one-bit mutation once to each individual. Let the population size be $N=n+1$. Assume that the current population $P_t$ covers $|f(P_t)| \ge (1-\eps)(n+1)$ elements of the Pareto front. Then with probability at least $1-\exp(-\Omega(n))$, the next population $P_{t+1}$ covers less than $(1-0.01)(n+1)$ elements of the Pareto front.
\label{lem:wPf}
\end{lemma}

\begin{proof}
Let $U$ be the set of Pareto front values that have exactly one corresponding individual in $P_t$, that is, for any $(v,n-v) \in U$, there exists only one $x \in P_t$ with $f(x) = (v,n-v)$. We first note that $|U| \ge (1-2\eps)(n+1)$ as otherwise there would be at least
\begin{align*}
2&{}\left(|f(P_t)| - |U|\right) + |U| = 2|f(P_t)|-|U| \\
&>{} 2(1-\eps)(n+1) -(1-2\eps) (n+1)= n+1
\end{align*} 
individuals in $P_t$, which contradicts our assumption $N=n+1$. 

Let $U'$ denote the set of values in $U$ which have all their neighbors also in $U$. 
Since each value not in $U$ can prevent at most two values in $U$ from being in $U'$, we have
\begin{equation*}
\begin{split}
|U'| &\ge{} |U| - 2(n+1-|U|) = 3|U| -2(n+1) \\
&\ge{} 3(1-2\eps)(n+1) - 2(n+1)
= \left(1-6\eps\right)(n+1).
\label{eq:ur'}
\end{split}
\end{equation*} 
We say that $(v,n-v)$ is double-covered by $R_t$ if there are exactly two individuals in $R_t$ with function value $(v,n-v)$. Noting that via one-bit mutation a certain function value can only be generated from the individuals corresponding to the neighbors of this function value, we see that a given $(v,n-v) \in U'$ with $v \in [1..n-1]$ is double-covered by $R_t$ with probability exactly 
$$
p_v = \frac{n-(v-1)}{n}+\frac{v+1}{n}-2\frac{n-(v-1)}{n}\frac{v+1}{n}=1-\frac{2v}{n}+2\frac{v^2-1}{n^2}.
$$
Thus the expected number of double-coverages in $U'$ is at least
\begin{align*}
\sum_{v \in [1..n-1]: \atop (v,n-v)\in U'} p_v 
={}& \left(\sum_{v=1}^{n-1} p_v\right) -\left(\sum_{v \in [1..n-1] : \atop (v,n-v)\notin U'}  p_v\right)\\
\ge{}& \left(\sum_{v=1}^{n-1}1-\frac{2v}{n}+2\frac{v^2-1}{n^2}\right) -\left(\sum_{v \in [1..n-1]: \atop (v,n-v)\notin U'} 1\right)\\
\ge{}& (n-1)-\frac{2}{n}\frac{(n-1)n}{2}+\frac{2}{n^2}\left(\frac{(n-1)n(2(n-1)+1)}{6}-(n-1)\right)\\
&{}-6\eps(n+1)\\
= {}& \frac{n-1}{n^2}\frac{2n^2-n-6}{3}-6\eps(n+1) 
= (\tfrac 23 - 6\eps) (n+1) - O(1).
\end{align*}

Denote by $U''$ the set of values in $U'$ that are double-covered by $R_t$ and note that we have just shown $E[|U''|] \ge (\frac 23 - 6\eps)(n+1) - O(1)$. The number $m := |U''|$ of double-covered elements is functionally dependent on the random decisions taken (independently) in the $N$ mutation operations. Each mutation operation determines one offspring and thus can change the number of double-covered values by at most $2$. Consequently, we can use the method of bounded differences~\cite{McDiarmid89} and obtain that $|U''|$ is at least $(\tfrac 23 - 8 \eps) (n+1)$ with probability at least $1 - \exp(-\Omega(n))$. We condition on this in the remainder.

Our next argument is that these double-coverages correspond to approximately $\frac 43 (n+1)$ individuals in~$R_t$ that have the same crowding distance. Consequently, the selection procedure has to discard at least roughly $\frac 13 (n+1)$ of them, randomly chosen, and this will lead to a decent number of values in $U''$ that are not covered anymore by $P_{t+1}$. 

To make this precise, let $R''$ denote the individuals $x$ in $R_t$ such that $f(x) \in U''$. By construction, there are exactly two such individuals for each value in $U''$, hence $|R''| = 2m$. Further, both neighboring values are also present in $f(R_t)$. Consequently, each $x \in R''$ has crowding distance  (in $R_t$) exactly $d = \frac{1}{v_1^{\max}-v_1^{\min}} + \frac{1}{v_2^{\max}-v_2^{\min}}$. We recall that the selection procedure (since all ranks are equal to one) first discards all individuals with crowding distance less than $d$ since these are at most $|R_t| - |R''| \le 2(n+1-\tilde m) = (2 - \frac 43 + 16\eps) (n+1) + O(1)$, which is less than $N$ for $n$ large and $\eps$ small enough. Then, randomly, the selection procedure discards a further number of individuals from all individuals with crowding distance exactly $d$ so that exactly $N$ individuals remain. For $N$ individuals to remain, we need that at least  $k := |R''| - N$ individuals from $R''$ are discarded.

To ease the calculation, we first reduce the problem to the case that  
$|R''| = 2 \tilde m$. Indeed, let $U'''$ be any subset of $U''$  having 
cardinality exactly $\tilde m$ and let $R'''$ be the set of individuals $x \in R_t$ with $f(x) \in U'''$. Then $R''' \subseteq R''$ and $|R'''| = 2\tilde m$. With the same argument as in the previous paragraph we see that the selection procedure has to remove at least $\tilde k := 2\tilde 
m - N$ elements from $R'''$. We thus analyze the number of elements of 
$U'''$ that become uncovered when we remove a random set of $\tilde k$ 
individuals from $R'''$, knowing that this is a lower bound for the 
number of elements uncovered in $U''$, both because the number of 
individuals removed from $R'''$ can be higher than $\tilde k$ and 
because the removal of elements in $R'' \setminus R'''$ can also 
lead to uncovered elements in $U''$. 

We take a final pessimistic simplification, and this is that we select $\tilde k$ elements from $R'''$ with replacement and remove these individuals from $R'''$. Clearly, this can only lower the number of removed elements, hence our estimate for the number of uncovered elements is also valid for the random experiment without replacement (where we choose exactly $\tilde k$ elements to be removed).  

For this random experiment the probability for uncovering a position in $U'''$ is at least 
\begin{align*}
1-{}&{}2\left(1-\frac{1}{2\tilde m}\right)^{\tilde k}+\left(1 - \frac{1}{2\tilde m}\right)^{2\tilde k} 
\\
&={} 1-2\exp\left(-\frac{\tilde k}{2 \tilde m}\right) + \exp\left(-\frac{\tilde k}{\tilde m}\right) - O\left(\frac 1 n\right) \\
&\ge{} 1- 2\exp\left(-1+\frac 34 \frac{1}{1-12\eps}\right) + \exp\left(-2 + \frac 64\frac{1}{1-12\eps}\right) - O\left(\frac 1n\right) :=p,
\end{align*}
where we used the estimate $\frac{\tilde k}{2\tilde m} = 1 - \frac{n+1}{2\tilde m} = 1 - \frac 34 \frac{1}{1-12\eps}$ and the fact that $\tilde m = \Theta(n)$.

Let $Y$ denote the number of elements of $U'''$ uncovered in our random experiment. We note that $1 - 2\exp(-1/4) + \exp(-1/2) \ge 0.04892$. Hence when $n$ is large enough and $\eps$ was chosen as a sufficiently small constant, then  
\[
E[Y] = p \tilde{m} \ge  0.02 (n+1).
\]
The random variable $Y$ is functionally dependent on the $\tilde k$ selected individuals, which are stochastically independent. Changing the outcome of a single selected individual changes $Y$ by at most $1$. Consequently, $Y$ satisfies the assumptions of the method of bounded differences~\cite{McDiarmid89}. The classic additive Chernoff bound thus applies to $Y$ as if it was a sum of $k=\Omega(n)$ independent random variables taking values in an interval of length~$1$. In particular, the probability that $Y \le 0.01 (n+1) \le \frac12 E[Y]$ is at most $\exp(-\Omega(n))$.
\end{proof}

Combining Lemmas~\ref{lem:Rt} and~\ref{lem:wPf}, we have the following exponential runtime result.
\begin{theorem} 
Consider optimizing \omm via the \mbox{NSGA-II} applying one-bit mutation once to each individual. Let the population size be $N=n+1$. There are a positive constant $\gamma$ and a time $T = \exp(\Omega(n))$ such that with probability $1 - \exp(-\Omega(n))$, in each of the first $T$ iterations at most a fraction of $1-\gamma$ of the Pareto front is covered by $P_t$.
\label{thm:lb}
\end{theorem}

\begin{proof}
  Let $\eps$ be a small constant rendering the claims of Lemmas~\ref{lem:Rt} and~\ref{lem:wPf} valid. Assume that $n$ is sufficiently large. Let $\tilde \eps = (\frac 1 {10} \eps)^{5/\eps+1}$. By a simple Chernoff bound, we note that a random initial individual~$x$ satisfies $\frac 14n \le f_1(x) \le \frac 34 n$ with probability $1 - \exp(\Omega(n))$. Taking a union bound over the $n+1$ initial individuals, we see that the initial population $P_0$ with probability $1 - \exp(-\Omega(n))$ covers at most half of the Pareto front. 
	
	Let $t$ be some iteration. 	If $|f(P_t)| \ge (1-\eps)(n+1)$, then by Lemma~\ref{lem:wPf} with probability $1 - \exp(-\Omega(n))$ the next population $P_{t+1}$ covers less than $(1 - 0.01)(n+1)$ values of the Pareto front. If $|f(P_t)| \le (1-\eps) (n+1)$, then by Lemma~\ref{lem:Rt} with probability $1 - \exp(-\Omega(n))$ we have $n+1 - |f(P_{t+1})| \ge \frac 1 {10} \eps (\frac15\eps - \tfrac 2n)^{5/\eps} (n+1) \ge \tilde \eps (n+1)$, where the last estimate holds when $n$ is sufficiently large.
	Consequently, for each generation $t$, the probability that $P_t$ covers more than $(1 - \min\{\tilde \eps, 0.01\}) (n+1)$ values of the Pareto front, is only $\exp(-\Omega(n))$. In particular, a union bound shows that for $T = \exp(\Theta(n))$ suitably chosen, with probability $1 - \exp(-\Omega(n))$ in all of the first $T$ iterations, the population covers at most $(1 - \min\{\tilde \eps, 0.01\}) (n+1)$ values of the Pareto front.
\end{proof}

\section{Experiments}
\label{sec:exp}

To complement our asymptotic results with runtime data for concrete problem sizes, we conducted the following experiments.

\subsection{Settings}

We use, in principle, the version of the \mbox{NSGA-II}  given by Deb (Revision~1.1.6), available at~\cite{DebNSGAII}, except that, as in our theoretical analysis, we do not use crossover. We re-implemented the algorithm in Matlab (R2016b). When a sorting procedure is used, we use the one provided by Matlab (and not randomized Quicksort as in Deb's implementation). The code is available at~\cite{ZhengNSGAII}.

Our theoretical analysis above covers four parent selection strategies and two mutation operators. In the interest of brevity, with the exception of the data presented in Figure~\ref{fig:ratiosVars} we concentrate in our experiments on one variant of the algorithm, namely we use two-permutation binary tournament selection (as proposed in~\cite{DebPAM02}) and standard bit-wise mutation with mutation rate~$\frac 1n$ (which is the most common mutation operator in evolutionary computation). 
We use the following experimental settings.
\begin{itemize}
\item Problem size $n$: $100,200,300,$ and $400$ for \omm, and $30,60,90,$ and $120$ for \lotz.
\item Population size $N$: Our theoretical analyses (Theorems~\ref{thm:ommbinary} and~\ref{thm:lotzbinary}) showed that the \mbox{NSGA-II} find the optima of \omm and \lotz efficiently for population sizes of at least $N^* = 4(n+1)$. We use this value also in the experiments. We also use the value $N = 2N^*$, for which our theory results apply, but our runtime guarantees are twice as large as for $N^*$ (when making the implicit constants in the results visible). We also use the smaller population sizes $2(n+1)$ and $1.5(n+1)$ for \omm and $2(n+1)$ for \lotz. For these values, we have no proven result, but it is not uncommon that mathematical runtime analyses cannot cover all efficient parameter setting, and in fact, we shall observe a good performance in these experiments as well (the reason why we do not display results for $N = 1.5(n+1)$ for \lotz is that here indeed the algorithm was not effective anymore). Finally, we conduct experiments with the population size $N = n+1$, which is large enough to represent the full Pareto front, but for which we have proven the \mbox{NSGA-II} to be ineffective (on \omm and when letting each parent create an offspring via one-bit mutation).
\item Number of independent runs: $50$ for the efficient population sizes in Section~\ref{subsec:eff} and $20$ for more time-consuming experiments with inefficient population sizes in Sections~\ref{subsec:ine} to~\ref{subsec:small}. These numbers of independent runs have already shown good concentrations.
\end{itemize}

\subsection{Efficient Population Sizes}\label{subsec:eff}

Figure~\ref{fig:runtimeomm} displays the runtime (that is, the number of fitness evaluations until the full Pareto front is covered) of the \mbox{NSGA-II} with population sizes large enough to allow an efficient optimization, together with the runtime of the (parameter-less) GSEMO. 

\begin{figure}[!ht]
\centering
\subfloat[\label{fig:ra}]{
\includegraphics[width=4.8in]{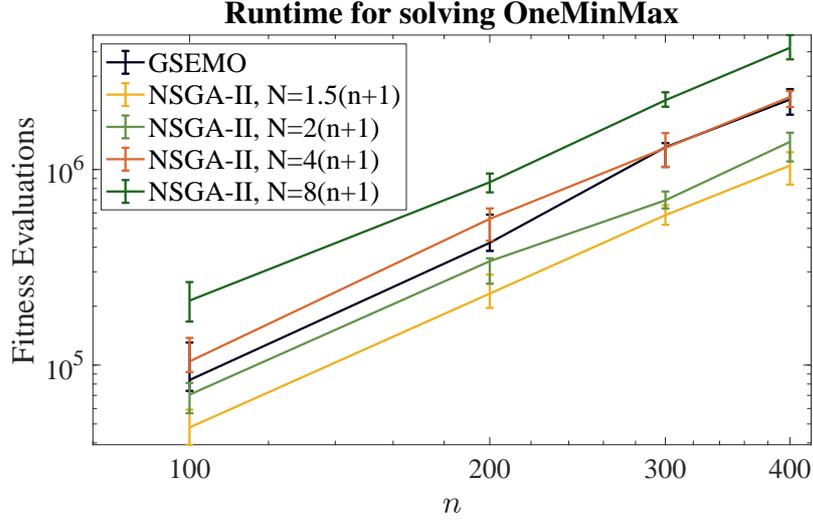}}\\
\subfloat[\label{fig:rb}]{\includegraphics[width=4.8in]{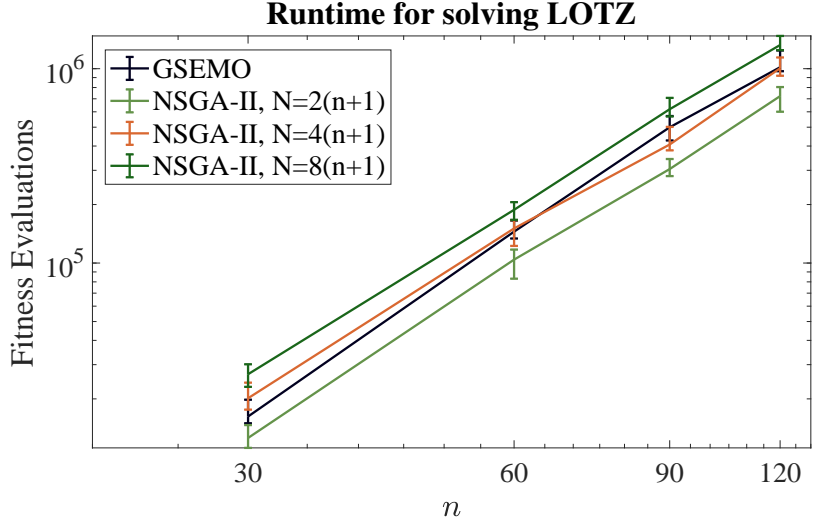}} 
\caption{The number of function evaluations for the \mbox{NSGA-II} (binary tournament selection, standard bit-wise mutation) with different population sizes and for the GSEMO optimizing \omm (\ref{fig:ra}) and  \lotz (\ref{fig:rb}). Displayed are the median (with $1$st and $3$rd quartiles) in 50 independent runs. }
\label{fig:runtimeomm}
\end{figure}

This data confirms that the \mbox{NSGA-II} can efficiently cover the Pareto front of \omm and \lotz when using a population size of at least $N^*$. The runtimes for $N = 2N^*$ are clearly larger than for $N^*$,  but by a factor slightly less than $2$ for both problems. The data for the population sizes smaller than $N^*$ indicates that also for these parameter settings the \mbox{NSGA-II} performs very well. 

Comparing the \mbox{NSGA-II} to the GSEMO, we observe that the \mbox{NSGA-II} with a proper choice of the population size shows a better performance. This is interesting and somewhat unexpected, in particular, for simple problems like \omm and \lotz. It is clear that the \NSGA using tournament selection chooses extremal parents with higher rate. More precisely, each individual appears twice in a tournament. For an extremal value on the Pareto front, at least one individual has an infinite crowding distance, making it the tournament winner almost surely (except in the rare case that the tournament partner has infinite crowding distance as well). Consequently, for each extremal objective value, the \NSGA mutates at least $2 - o(1)$ individuals per iteration. This is twice the average rate. In contrast, the GSEMO treats all individuals equally. This advantage of the \NSGA comes at the price of a larger population, hence a larger cost per iteration. We note that the \mbox{NSGA-II} throughout the run works with a population of size~$N$, whereas the GSEMO only keeps non-dominated individuals in its population. Consequently, in particular in the early stages of the optimization process, each iteration takes significantly fewer fitness evaluations.

\subsection{Inefficient Population Sizes}\label{subsec:ine}

When the population size is small, we do not have the result that points on the front cannot be lost (Lemmas~\ref{lem:keep} and~\ref{lem:keeplotz}) and the proof of 
Theorem~\ref{thm:lb} shows that indeed we can easily lose points on the front, leading to a runtime at least exponential in~$n$ when $N= n+1$. In this subsection, we analyze this phenomenon experimentally. As discussed earlier, we first concentrate on the \mbox{NSGA-II} with two-permutation tournament selection and standard bit-wise mutation. 

Since it is hard to show an exponential runtime experimentally, we do not run the algorithm until it found the full Pareto front (this would be possible only for very small problem sizes), but we conduct a slightly different experiment for reasonable problem sizes which also strongly indicates that the \mbox{NSGA-II} has enormous difficulties in finding the full front. We ran the \mbox{NSGA-II} for $3000$ generations for \omm and $5000$ generations for \lotz and measured for each generation the ratio by which the Pareto front is covered. This data is displayed in  Figure~\ref{fig:ratiosomm}. We see clearly that the coverage of the Pareto front steeply increases at first, but then stagnates at a constant fraction clearly below one (around $80$\% for \omm and between 50\% and 60\% for \lotz) and this in a very concentrated manner. From this data, there is no indication that the Pareto front will be covered anytime soon. 

\begin{figure}[!ht]
\centering
\subfloat[\label{fig:raa}]{
\includegraphics[width=4.8in]{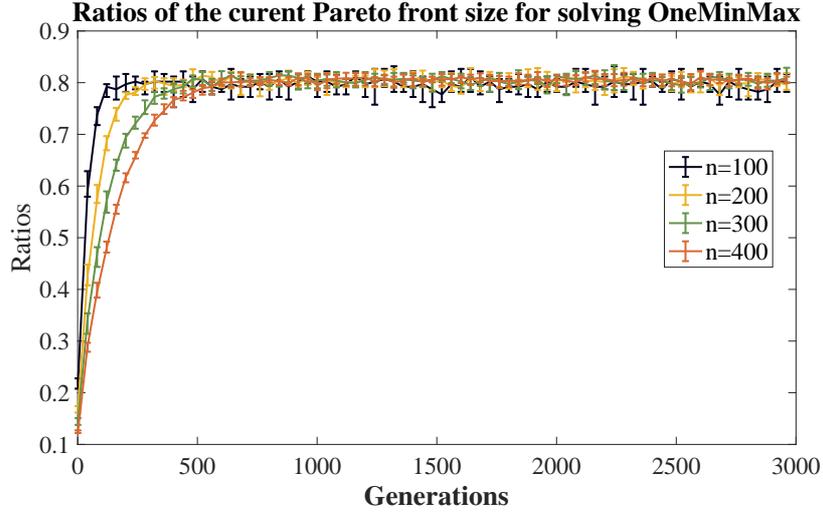}} \\
\subfloat[\label{fig:rab}]{\includegraphics[width=4.8in]{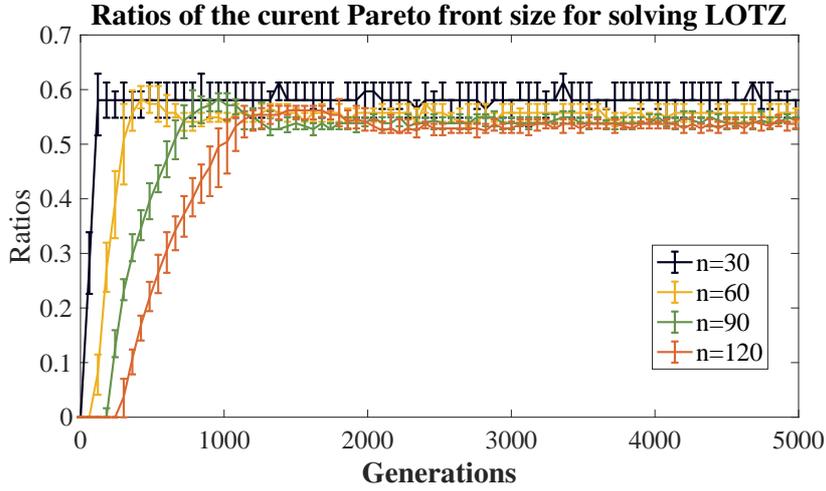}} 
\caption{Ratio of the coverage of the Pareto front by the current population of the \mbox{NSGA-II} (binary tournament selection, standard bit-wise mutation) with population size $N=n+1$ for solving \omm (\ref{fig:raa}) and \lotz (\ref{fig:rab}). Displayed are the median (with $1$st and $3$rd quartiles) in 20 independent runs.}
\label{fig:ratiosomm}
\end{figure}

We said in Section~\ref{sec:lb} that we were optimistic that our negative result for small population size would also hold for all other variants of the \mbox{NSGA-II}. To experimentally support this claim, we now run all variants of the \mbox{NSGA-II} discussed in this work on \omm with problem size $n=200$, $20$ times for $3000$ iterations. In Figure~\ref{fig:ratiosVars}, we see the ratios of the coverage of the Pareto front by the populations in the $20$ runs and in iterations $[2001..3000]$ (that is, we regard together $20*1000$ populations). We see that all variants fail to cover a constant fraction of the Pareto. The precise constant is different for each variant. Most notable, we observe that the variants using standard bit-wise mutation cover the Pareto front to a lesser extent than those building on one-bit mutation. We do not have a clear explanation for this phenomenon, but we speculate that standard bit-wise mutation is harmed by its constant fraction of mutations that just create a copy of the parent. We would, however, not interpret the results in this figure as a suggestion to prefer one-bit mutation. As shown in~\cite{DoerrQ23tec}, with high probability the \mbox{NSGA-II} using one-bit mutation fails to find the Pareto front of the \ojzj benchmark, regardless of the runtime allowed.  

\begin{figure}[!ht]
\centering
\includegraphics[width=4.8in]{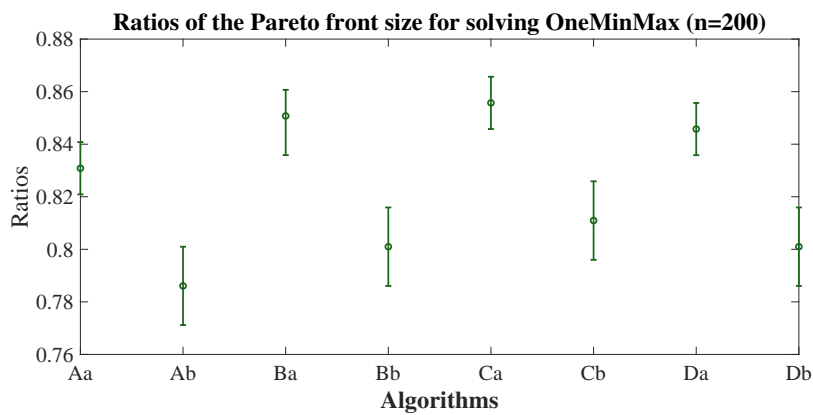} 
\caption{Ratios of the coverage of the Pareto front by the population of the different \mbox{NSGA-II} variants (using $A$ (selecting each individual as a parent once), $B$ ($N$ times choosing a parent uniformly at random), $C$ (independent binary tournaments), or $D$ (two-permutation binary tournaments) as the mating selection strategy, and using $a$ (one-bit mutation) or $b$ (standard bit-wise mutation) as the mutation strategy) with population size $N=n+1$ on the \omm with problem size $n=200$. Displayed are the median (with $1$st and $3$rd quartiles) in 20 independent runs and $[2001..3000]$ generations.}
\label{fig:ratiosVars}
\end{figure}

\subsection{Optimization With Small Population Sizes}\label{subsec:small}

In the previous subsection, we showed that the \NSGA with population size equal to the size of the Pareto front cannot cover the full Pareto front in a reasonable time. On the positive side, however, still a large fraction of the Pareto front was covered, e.g., around 80\% for the \omm problem. This could indicate that the \NSGA also with smaller population sizes is an interesting algorithm. This is what we brief{}ly discuss now. We shall not explore this question in full detail, but only to the extent that we observe a good indication that the \NSGA performs well also with small population sizes. We note that the subsequent work~\cite{ZhengD22gecco} took up this research question and discussed it in detail.

To understand how well the \NSGA performs with small population size $n+1$, we first regard how fast its population spreads out on the Pareto front. From the data in Figure~\ref{fig:genTwoEnds}, we see that also with this small population size, the \NSGA quickly finds the two extremal points $(0,n)$ and $(n,0)$ of the Pareto front. This fits our understanding of the algorithms. Since the two outer-most individuals in the population have infinite crowding distance and since there are at most four individuals with infinite crowding distance, these individuals will never be lost, even if the population size is relatively small. 

More interesting is the question how evenly the population is distributed on the Pareto front once the two extremal points are found. To this aim, we display in Figure~\ref{fig:ParetoCover} the function values of the populations after a moderate runtime in a run of the \mbox{NSGA-II}. In all eight datasets, the complete Pareto front was not found (as expected). However, the plots also show that in all cases, the front is well approximated by the population. 
Also, we note that the population contains only individuals on the Pareto front (which is trivially satisfied for \omm, but not so for \lotz). We note that the data from two individual runs displayed in the figure is representative. In all runs we never
encountered an interval of uncovered points of length longer than $6$ and $4$ respectively.

\begin{figure}[!ht]
\centering
\subfloat[\label{fig:ga}]{
\includegraphics[scale=0.28]{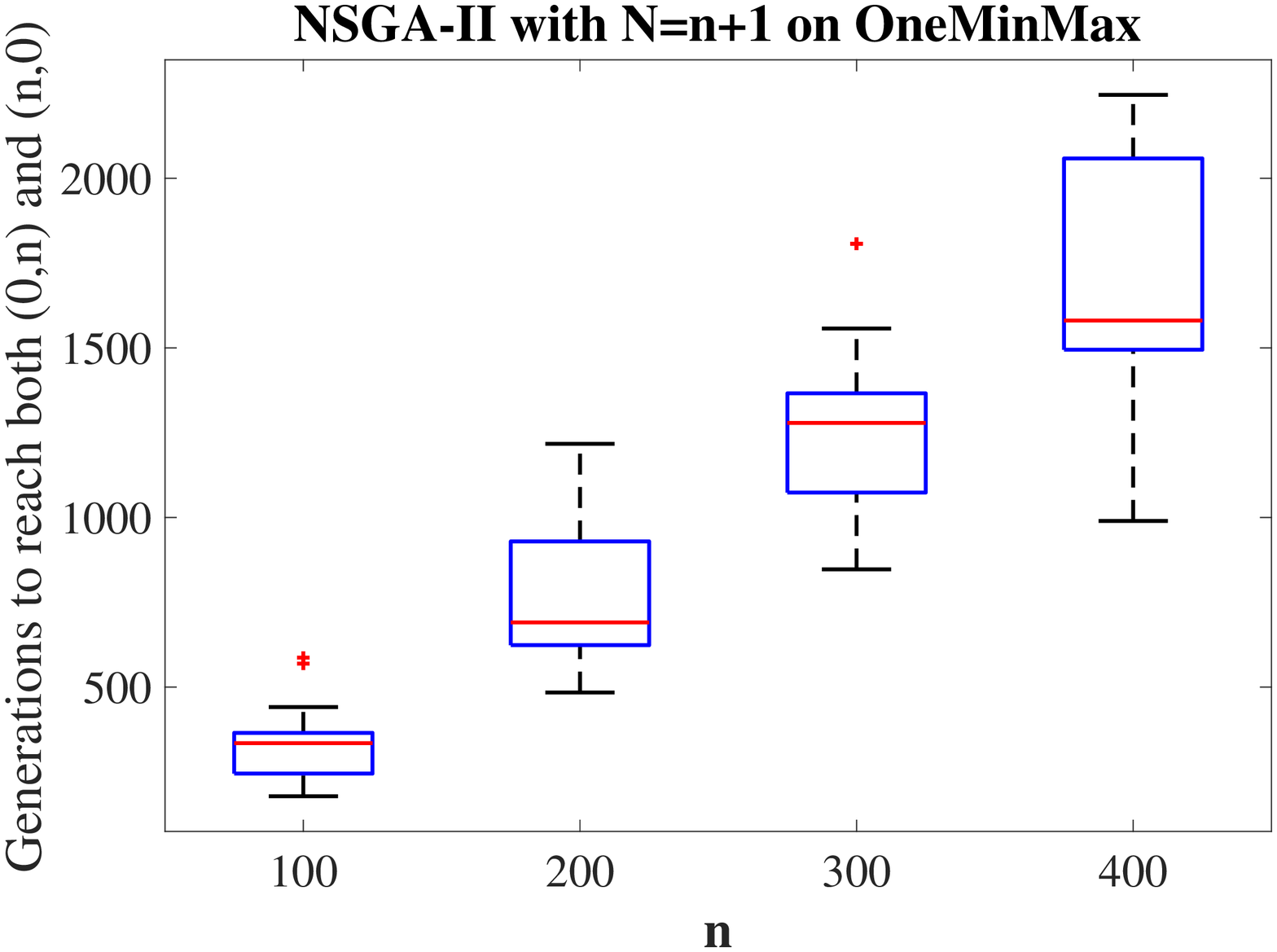}} 
  \hspace{0.05in}
  \subfloat[\label{fig:gb}]{\includegraphics[scale=0.29]{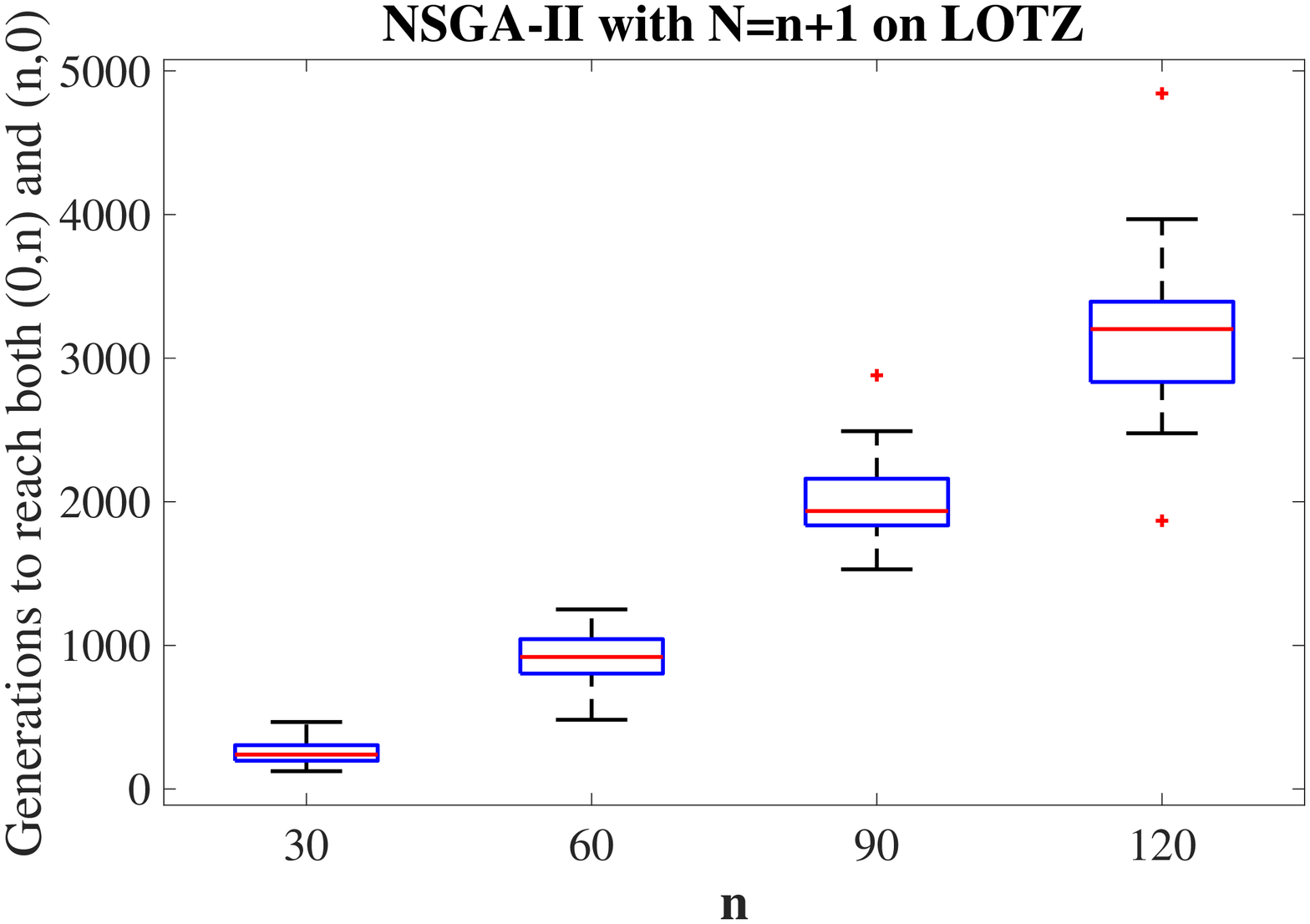}} 
\caption{First generation when both extreme function values $(0,n)$ and $(n,0)$ were contained in the population of the \mbox{NSGA-II} (binary tournament selection, standard bit-wise mutation, population size $N=n+1$) for \omm (\ref{fig:ga}) and \lotz (\ref{fig:gb}). }
\label{fig:genTwoEnds}
\end{figure}

\begin{figure}[!ht]
  \centering
  \subfloat[\label{fig:pa}]{\includegraphics[scale=0.35]{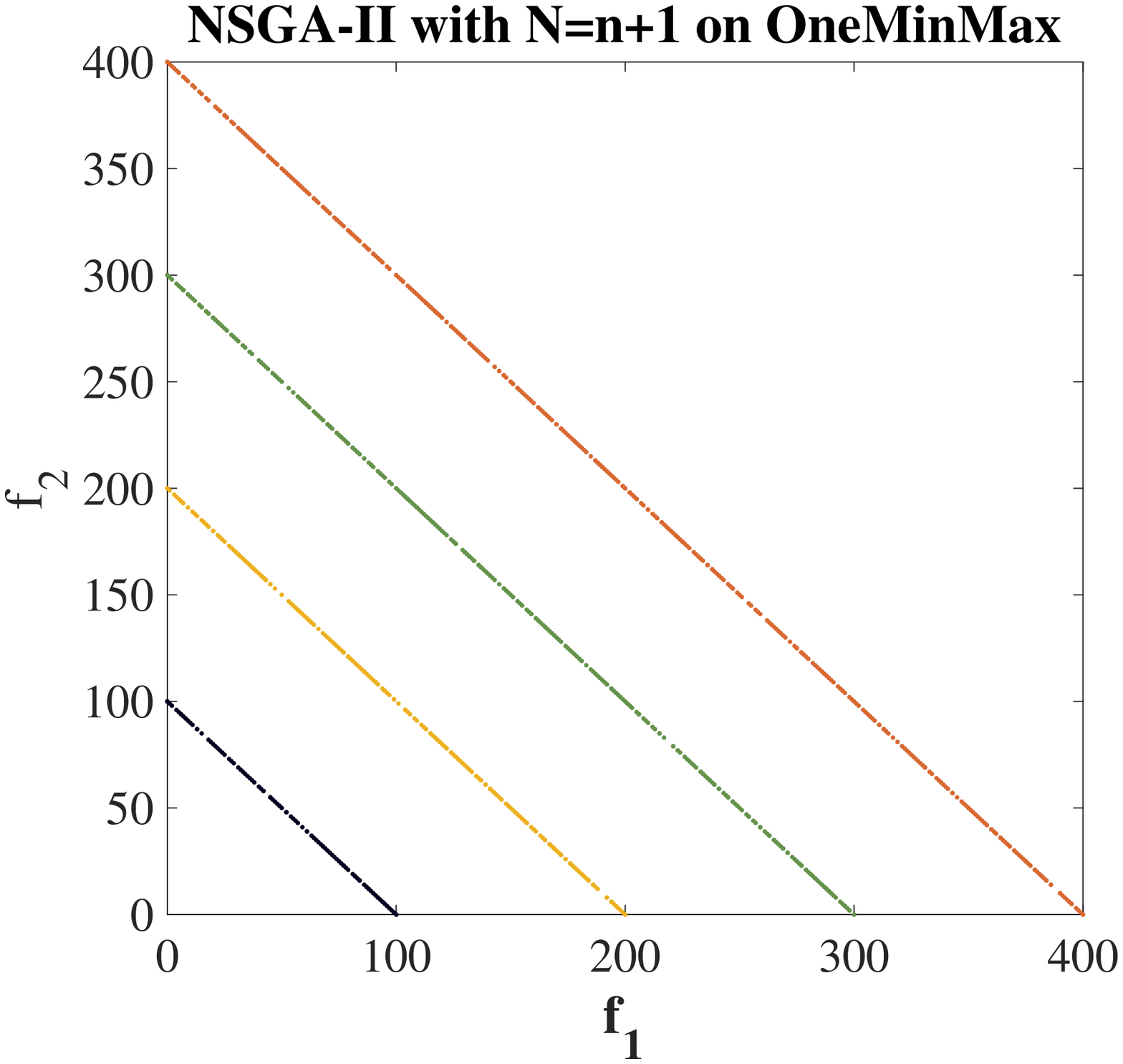}}
  \hspace{0.05in}
  \subfloat[\label{fig:pb}]{\includegraphics[scale=0.35]{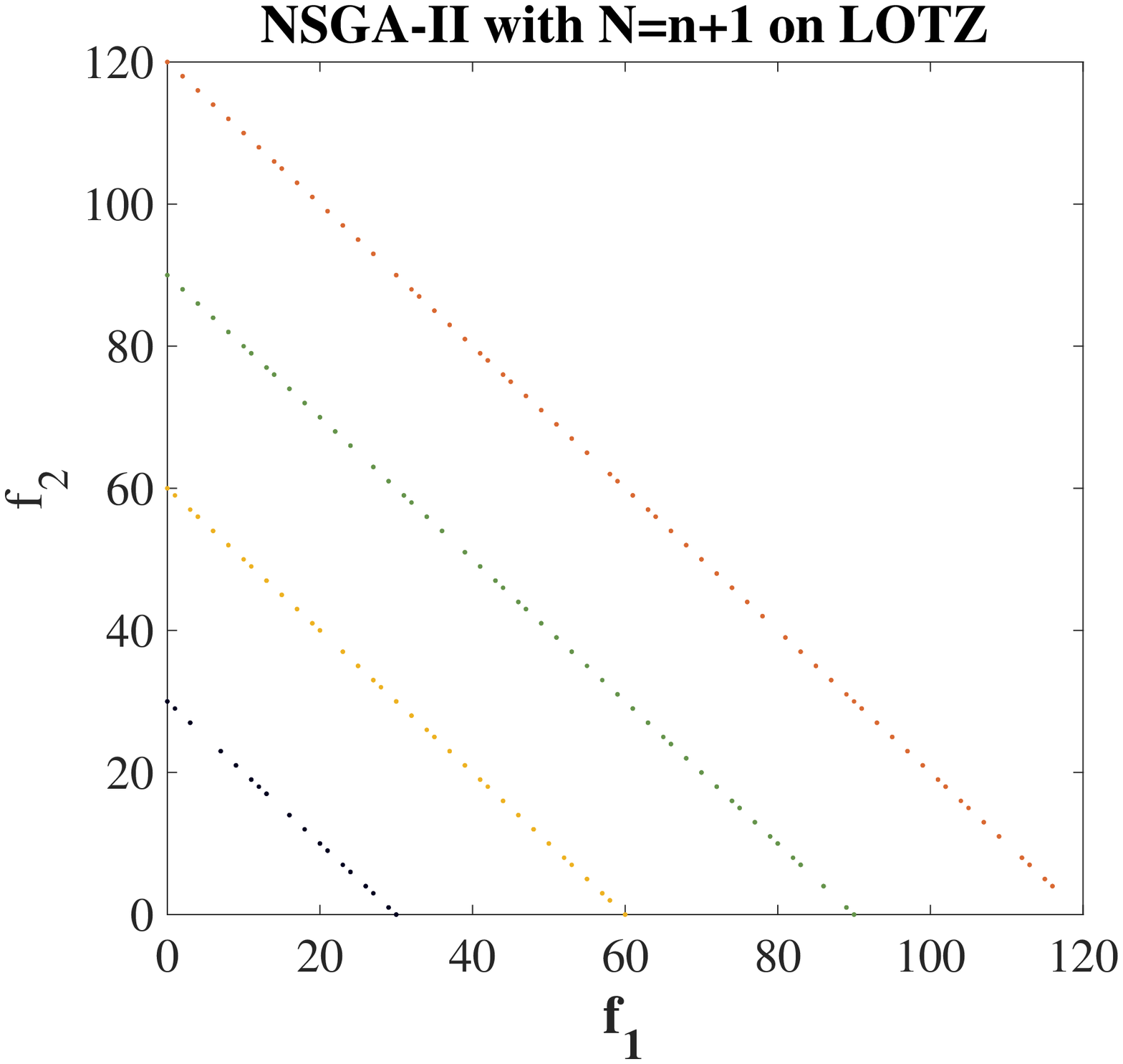}}
  \caption{The function values of the population $P_t$ for $t=3000$ when optimizing \omm (\ref{fig:pa}) and for $t=5000$ when optimizing \lotz (\ref{fig:pb})  via the \mbox{NSGA-II} (binary tournament selection, standard bit-wise mutation, population size $N=n+1$) in one typical run. Both plots show that this population size is not sufficient to completely cover the Pareto front, but it suffices to approximate very well the front. Different colors are for different problem sizes $n$, and $n=\{100,200,300,400\}$ for \omm and $n=\{30,60,90,120\}$ for \lotz. Also note that the Pareto front is $\{(i,n-i)\mid i\in[0..n]\}$.}
  \label{fig:ParetoCover}
\end{figure}

\section{Conclusion}\label{sec:con}

In this work, we conducted the first mathematical runtime analysis of the \mbox{NSGA-II}, which is the predominant framework in real-world multi-objective optimization. We proved that with a suitable population size, all variants of the \mbox{NSGA-II} regarded in this work satisfy the same asymptotic runtime guarantees as the previously regarded much simpler SEMO, GSEMO, and \sibea when optimizing the two benchmarks \omm and \lotz. The choice of the population size is important. We proved
 an exponential runtime when the population size equals the size of the Pareto front.

On the technical side, this paper shows that mathematical runtime analyses are feasible also for the \mbox{NSGA-II}. We provided a number of arguments to cope with the challenges imposed by this algorithm, in particular, the fact that points in the Pareto front can be lost and the parent selection via binary tournaments based on the rank and crowding distance. We are optimistic that these tools will aid future analyses of the \mbox{NSGA-II} (and in fact, they have already been used several times in subsequent work, see the discussion in the introduction).

\section*{Acknowledgments}
This work was supported by National Natural Science Foundation of China (Grant No. 62306086), Science, Technology and Innovation Commission of Shenzhen Municipality (Grant No. GXWD20220818191018001), Guangdong Basic and Applied Basic Research Foundation (Grant No. 2019A1515110177).

This work was also supported by a public grant as part of the Investissement d'avenir project, reference ANR-11-LABX-0056-LMH, LabEx LMH.

\newcommand{\etalchar}[1]{$^{#1}$}

\end{document}